%% file: main.tex
\documentclass{article}

\input{preamble}
\usepackage[square,numbers]{natbib}
    \usepackage[preprint]{neurips_2025}

\usepackage[utf8]{inputenc} 
\usepackage[T1]{fontenc}    
\usepackage{hyperref}       
\usepackage{url}            
\usepackage{booktabs}       
\usepackage{amsfonts}       
\usepackage{nicefrac}       
\usepackage{microtype}      
\usepackage{xcolor}         

\title{Spectral Clustering for Directed Graphs via Likelihood Estimation on Stochastic Block Models}

\author{
\textbf{Ning Zhang}\\
Department of Statistics\\
University of Oxford \\
\texttt{ning.zhang@stats.ox.ac.uk}
\And
\textbf{Xiaowen Dong}\\
Department of Engineering Science\\
University of Oxford \\
The Alan Turing Institute\\
\texttt{xdong@robots.ox.ac.uk}
\And \textbf{Mihai Cucuringu}\\
Department of Mathematics\\
University of California, Los Angeles \\
University of Oxford\\
\texttt{mihai@math.ucla.edu}
}

\begin{document}

\maketitle

\begin{abstract}
   Graph clustering is a fundamental task in unsupervised learning with broad real-world applications. 
   While spectral clustering methods for undirected graphs are well-established and guided by a minimum cut optimization consensus, their extension to directed graphs remains relatively underexplored due to the additional complexity introduced by edge directions. 
   In this paper, we leverage statistical inference on stochastic block models to guide the development of a spectral clustering algorithm for directed graphs. Specifically, we study the maximum likelihood estimation under a widely used directed stochastic block model, and derive a global objective function that aligns with the underlying community structure.
    We further establish a theoretical upper bound on the misclustering error of its spectral relaxation, and based on this relaxation, introduce a novel, self-adaptive spectral clustering method for directed graphs. Extensive experiments on synthetic and real-world datasets demonstrate significant performance gains over existing baselines.
\end{abstract}

\section{Introduction}
\label{sec:introduction}

\input{sections/1_intro}

\section{Preliminaries}
\label{sec:preliminaries}

\input{sections/2_preliminaries}

\section{Methodology: MLE on DSBMs}
\label{sec:method}

\input{sections/methodology}

\section{Spectral relaxation of Herm-MLE: Algorithm and error analysis}
\label{sec:analysis}
\input{sections/analysis}
\section{Iterative algorithm and experiments}
\label{sec:experiment}

\input{sections/experiment}

\section{Discussion}
\label{sec:discussion}

\input{sections/discussion}

\bibliography{ref}

\newpage
\appendix
\section{Summary on notations}
\label{appx_sumNot}

\input{sections/appx_notations}

\section{Proof of MLE}
\label{pf-thm:MLE}

\input{sections/appx_proof_MLE}
\section{Proofs in perturbation analysis}
\label{appx_perturbation} 
\input{sections/appx_proof_perturb}
\section{More Experimental details}

\input{sections/appx_convergence}

\input{sections/appx_complexity}

\input{sections/appx_DSBM}

\input{sections/appx_real_data}

\end{document}

%% file: preamble.tex
\usepackage{graphicx}
\usepackage{color}
\usepackage{amssymb}
\usepackage{amscd}
\usepackage{bbm}
\usepackage{tikz}
\usetikzlibrary{patterns}
\usepackage{hyperref}
\usepackage{lipsum}
\usepackage{bookmark}

\hypersetup{
    colorlinks=true,       
    linkcolor=black,          
    citecolor=blue,        
    filecolor=magenta,      
    urlcolor=cyan           
}

\usepackage[utf8]{inputenc}

\usepackage{amsfonts}
\usepackage{bbm} 
\usepackage{comment}
\usepackage[ruled,linesnumbered]{algorithm2e}
\makeatletter
\newcommand{\nosemic}{\renewcommand{\@endalgocfline}{\relax}}
\newcommand{\dosemic}{\renewcommand{\@endalgocfline}{\algocf@endline}}
\let\oldnl\nl
\newcommand{\nonl}{\renewcommand{\nl}{\let\nl\oldnl}}
\makeatother

\usepackage{mathtools,amsmath}
\usepackage{amssymb}
\usepackage{enumerate}
\usepackage[inline]{enumitem}
\usepackage{multirow}
\usepackage{bbold}
\usepackage{bm}

\usepackage{caption}
\usepackage{subcaption}
\usepackage{cleveref}
\usepackage{thmtools}
\usepackage{amsthm}
\usepackage{thm-restate}
\newtheorem{theorem}{Theorem}
\newtheorem{lemma}{Lemma}

\theoremstyle{definition}
\newtheorem{definition}{Definition}

\usepackage{booktabs}
\usepackage{longtable}
\usepackage{tikz}
\usepackage{fixltx2e}
\usepackage[parfill]{parskip}
\usepackage{algorithm2e}
\usepackage{float}
\DeclareMathOperator*{\argmax}{arg\,max}

\newcommand{\N}{\mathbb{N}}
\newcommand{\R}{\mathbb{R}}
\newcommand{\C}{\mathbb{C}}
\newcommand{\E}{\mathbb{E}}
\newcommand{\var}{\mathrm{Var}}
\newcommand{\iu}{{i\mkern1mu}}

\newcommand{\NF}{\mathrm{\textbf{NF}}}
\newcommand{\TF}{\mathrm{\textbf{TF}}}
\newcommand{\Herm}{\mathcal{H}}

\newcommand{\prob}{\mathbb{P}}
\newcommand{\Hmle}{\widetilde{H}}
\newcommand{\pmax}{p_{\mathrm{max}}}
\newcommand{\topv}{\bold{\hat{v}}}
\newcommand{\topvE}{\bold{{v}}}
\newcommand{\diag}{\mathrm{diag}}
\newcommand{\tr}{\mathrm{Tr}}
\newcommand{\ca}{\mathcal{C}_1}
\newcommand{\cb}{\mathcal{C}_2}
\newcommand{\llin}{M_{intra}}
\newcommand{\llout}{M_{inter}}

\newcommand*{\blue}{\textcolor{blue}}
\newcommand*{\red}{\textcolor{red}}

\definecolor{colour1}{RGB}{166,206,227}
\definecolor{colour2}{RGB}{31,120,180}
\definecolor{colour3}{RGB}{178,55,250} 
\definecolor{colour4}{RGB}{51,160,44}

\newcounter{noteMCctr} 

\usepackage[normalem]{ulem}

%% file: sections/1_intro.tex
Graph clustering is a fundamental problem in unsupervised learning, providing tools for uncovering hidden structure in complex relational data from social networks \citep{oliveira-2012-overview}, economics \citep{bennett-2022-lead_lag}, and neuroscience \citep{priebe-2017-semiparametric}. In graph clustering, the goal is to group vertices into ``structurally equivalent'' communities, where vertices from the same community interact with the rest of the graph in similar ways. Among various approaches, spectral clustering has become one of the most popular methods due to its computational efficiency. Most spectral clustering methods focus on undirected graphs, building on well-established optimization objectives such as cut minimization \citep{shi-2000-normalized, mcsherry-2001-spectral, von-2007-tutorial} or Girvan-Newman modularity maximization \citep{newman-2013-spectral,newman-2016-equivalence}.

In many real-world networks, however, interactions are inherently directional as seen in causal relationships \citep{pearl-1987-dependency_digraph}, interbank debt \citep{acemoglu-2015-financial_net}, paper citations \citep{an-2004-citation}, and neuron synapses \citep{priebe-2017-semiparametric}.  In such settings, edge directionality carries important structural information and is often tightly correlated with community membership. For instance, in neuron-neuron interaction networks, synaptic direction typically depends on the types of pre- and post-synaptic cells \citep{eichler-2017-complete}. 
Such strong coupling between edge orientation and community membership is naturally modeled by a directed stochastic block model (DSBM) \citep{wang-1987-stochastic, cucuringu-2020-hermitian}, motivating our development of a model-based approach that leverages statistical inference on DSBMs to guide the design of clustering algorithms for directed graphs.

Existing spectral methods for clustering directed graphs are largely heuristic, relying on singular value decomposition (SVD) \citep{rohe-2016-coclustering, wang-2020-dscore}, symmetrization techniques \citep{satuluri-2011-symmetrizations}, or manually prespecified optimization objectives \citep{leicht-2008-community, fanuel-2017-magnetic, laenen-2020-higher, meilua-2007-clustering,hayashi-2022-skew}. 
Despite the empirical success of these approaches, it remains unclear whether the underlying optimization criteria are well-justified, or whether the detected communities reflect meaningful structure or mere artifacts. In contrast, clustering criteria for undirected graphs benefit from rich theoretical frameworks grounded in studying stochastic block models (SBMs) and their variants \citep{newman-2016-equivalence, abbe-2015-exact, bickel-2009-nonparametric,abbe-2015-exact,hajek-2016-exact_SDP, zhang-2014-scalable}. For directed graphs, comparable theoretically grounded studies remain underdeveloped. 
Motivated by this gap, we present the first model-inspired spectral clustering for directed graphs through likelihood estimation under a well-established DSBM. We summarize our main contributions as follows:
\begin{enumerate}[leftmargin=0pt]
    \item []\textit{A principled optimization criterion for directed community detection (Section~\ref{sec:method}).}  We study the DSBM\citep{wang-1987-stochastic, cucuringu-2020-hermitian} with a source-sink community structure, and derive the maximum likelihood estimator (MLE) for recovering the planted communities (Theorem~\ref{thm:MLE}). This model-based formulation provides a statistically grounded approach compared to existing heuristic methods. Moreover, we show that the MLE objective is equivalent to a joint optimization over edge density and edge orientation, yielding a flow-based intuition that motivates model-free generalizations of our methodology.
    \item[] \textit{Theoretical guarantee for spectral clustering DSBMs (Section~\ref{sec:analysis}).}
    We present the MLE objective as a Hermitian quadratic form, which naturally enables a spectral relaxation for efficient computation. 
    We establish, in Theorem~\ref{thm:errorHermSC}, a high-probability upper bound on the misclustering error of this spectral relaxation using tools from matrix perturbation theory and random matrix theory. Our result improves upon prior work \citep{cucuringu-2020-hermitian} by allowing inhomogeneous edge density $p\neq q$ and removing a restrictive eigengap assumption.
    
    \item[] \textit{A self-adaptive spectral clustering algorithm (Section~\ref{sec:experiment}).} Based on this MLE formulation, we develop \texttt{LE-SC}, a fast and self-adaptive spectral algorithm for clustering directed graphs.  While spectral relaxation MLE typically requires knowledge of model parameters, our method employs an iterative pseudo-likelihood estimation procedure \citep{gong-1981-pseudo,newman-2016-equivalence} to learn the parameters directly from data.
    Extensive experiments on synthetic and real-world data sets, including a directed neuron-neuron connectome and a weighted migration network, demonstrate competitive performance over baselines.\footnote{Code is available at: {\url{https://github.com/NingZhang-Git/LE_SC}}}
    \end{enumerate}
    
\textbf{Related work.} 
A number of classical spectral clustering methods for directed graphs rely on the singular vectors \citep{rohe-2016-coclustering,rohe-2012-coclustering, wang-2020-dscore} or eigenvectors of symmetrized matrix representations \citep{kessler-1963-bibliographic, malliaros-2013-dicluster_survey, satuluri-2011-symmetrizations, small-1973-co}, with underlying optimization on graph connectivity patterns.
For example, in spectral clustering on the bibliographic coupling matrix $AA^T$, the symmetrized matrix representation involves the number of shared offspring vertices.
Another more recent line of spectral methods adopts complex-valued Hermitian matrices to represent directed graphs for spectral clustering algorithms~\citep{cucuringu-2020-hermitian,fanuel-2017-magnetic,laenen-2020-higher}. 
In \citep{cucuringu-2020-hermitian}, the authors cluster directed {graphs} using eigenvectors of a Hermitian matrix, employing $\pm\iu$ to represent directed edges, with a flow optimization heuristic \citep{hayashi-2022-skew,laenen-2019-directed-thesis}.
\citet{laenen-2020-higher} proposes a different Hermitian matrix, where the $k$-th roots of unity are used to represent directed edges, leading to a higher-order flow optimization scheme. {However, most existing studies lack a solid theoretical underpinning of why and when a heuristic approach might work, and this is exactly the gap we address in this study through a rigorous model-based analysis.}

%% file: sections/2_preliminaries.tex
Let $G (\mathcal{V},\mathcal{E})$ be a directed graph on vertex set $\mathcal{V}$ and edge set $\mathcal{E}$.  For a pair of vertices $u,v \in \mathcal{V}$, we denote $u\rightsquigarrow v$ if there is an edge pointing from $u$ to $v$ and we denote $u \not\sim v$ if there is no edge between $u$ and $v$. 
A directed graph can be represented by its adjacency matrix $A \in \{0,1\}^{N\times N}$, where $A_{uv} = 1$ if and only if  $u\rightsquigarrow v$. 
For a Hermitian matrix $H$, it has $n$ real eigenvalues, and throughout this paper, the eigenvalues are consistently organized in descending order of magnitude, i.e., $|\lambda_1(H)| \geq  |\lambda_2(H)| \geq \cdots |\lambda_n(H)|$. 
We use $H^T$ to denote its transpose and $H^*$ to denote the conjugate transpose,  $H^* = \overline{H}^T$. 
We also employ 
{several commonly used matrices: we use $I$ to denote the identity matrix and $J$ to denote the square all-one matrix. 
This paper makes use of several matrix norms, namely we use $\|H\|$ to denote the spectral norm, and $\|H\|_F$ to denote the Frobenius norm. 
For ease of reference, we refer to the summary in Table~\ref{tab:sumNot} of notions used in this paper. 

\subsection{Directed Stochastic Block Model}
The canonical stochastic block model for directed graphs was introduced by \citet{holland-1981-exponential}, known as the $p_1$ model. In this paper, we study an instance of the $p_1$ model, the DSBM \citep{cucuringu-2020-hermitian}, which considers only directed (nonreciprocal) edges. Our motivation is that, in the context of clustering directed graphs, the main challenge lies in handling edge directionality, while clustering with undirected (reciprocal) edges is relatively well understood \citep{abbe-2018-community}. We note that this focus on directed edges does not preclude integration with undirected clustering methods; rather, it provides a complementary perspective that could be synthesized with existing undirected approaches in future hybrid models.

DSBMs with multiple communities involve intricate higher-order interactions between communities (see Figure~\ref{fig:k3} for an example), and explicitly modeling such interactions may introduce bias by enforcing specific structural assumptions. For generality and interpretability,  we begin with the most basic two-community setup, with a \textit{source} community $\ca$ of size $n_1$ and a \textit{sink} community $\cb$ of size $n_2$. While our theoretical analysis focuses on this basic setting, the derived clustering algorithm naturally extends to multi-community problems via iterative bipartitioning (see the Experiment section for details). 

In a two-community DSBM, the community memberships are treated as {fixed but unknown parameters (rather than as latent variables). We introduce a vector $\bold{\sigma}$ to indicate the community assignments, where $\sigma_u = \sigma_v$ if and only if vertices $u,v$ belong to the same community.
Given a directed graph adjacency matrix sampled from the {two-community} DSBM $(n_1,n_2,p,q,\eta)$, its entries are generated independently, conditioning on the community labelling $\sigma$ as follows: 

\begin{minipage}{0.45\textwidth}
$u,v \in \ca$ or $u,v\in \cb$:
\begin{align*} 
    \begin{cases*}
       A_{uv} =  1, A_{vu} = 0  &  \text{w.p. } $p/2$,\\
       A_{uv} =  0, A_{vu} = 1  & \text{w.p. } $p/2$,\\
       A_{uv} = A_{vu} = 0   &  \text{w.p. } $1- p$.
    \end{cases*};
\end{align*}
\end{minipage}
\hfill
\begin{minipage}{0.45\textwidth}
For $u\in \ca, v \in \cb$:
\begin{align*}
    \begin{cases*}
       A_{uv} =  1, A_{vu} = 0  &  \text{w.p. } $(1-\eta)q$ ,\\
      A_{uv} =  0, A_{vu} = 1 &  \text{w.p. } $\eta q$,\\
     A_{uv} = A_{vu} = 0  &  \text{w.p. } $1-q$.
    \end{cases*}
\end{align*}
\end{minipage}

Here, the parameter $p$ denotes the probability of forming an edge within communities, while $q$ represents the edge probability between communities. Within-community edges are assigned directions uniformly at random, making their orientation symmetric in expectation. In contrast, the directionality of inter-community edges is controlled by the parameter $\eta \in [0,0.5]$, where an edge from $\ca$ to $\cb$ occurs with probability $(1 - \eta)$. 
This model captures community structure through both heterogeneous edge densities and asymmetric, community-dependent edge directions.

%% file: sections/methodology.tex
\subsection{MLE on DSBMs}
Given a graph with adjacency matrix $A$ generated from the DSBM, our goal is to infer the community labels by finding the assignment that renders the observed graph topology most probable. Specifically, we formulate community detection as a maximum likelihood estimation (MLE) problem: 
\begin{align*}
    \hat{\sigma}_{\mathrm{MLE}} =  \argmax_{\sigma} \mathcal{L}(A;\sigma)
\end{align*}
where the log-likelihood function is given by 
\begin{align*}
    \mathcal{L}(A;\sigma) = \sum_{u<v} \log( \prob(A_{u,v}|\sigma_u,\sigma_v))
\end{align*}
In Theorem~\ref{thm:MLE}, we derive a more compact representation of the MLE objective, showing that it can be reformulated as a quadratic form involving a Hermitian matrix. This formulation naturally motivates spectral approaches for community recovery. The full proof of Theorem~\ref{thm:MLE} is provided in Appendix~\ref{pf-thm:MLE}.

\begin{restatable}[MLE on DSBMs]{theorem}{ThmHermMLE}
\label{thm:MLE}
        Let $A$ be the adjacency matrix of a directed graph sampled from the DSBM$(n_1, n_2, p, q, \eta)$. Define the indicator vector $\mathbf{x} \in \{1, \iu\}^N$ such that $\mathbf{x}_u = \iu$ if $u \in \ca$ and $\mathbf{x}_u = 1$ if $u \in \cb$.
        Then, the MLE on the community labels is equivalent to the following optimization problem:
        \begin{align}
             \max \quad & \bold{x}^*\Hmle \bold{x} \tag{Herm-MLE}\label{opt:Herm-MLE}\\
             \nonumber
            s.t. \quad &{\bold{x} \in \{1,\iu\}^N}, 
        \end{align}
        where $\Hmle$ is a Hermitian matrix given by
        \begin{align}
        \label{eq:H-mle}
             \Hmle 
             & = i\log{\frac{1-\eta}{\eta}}\left(A -A^T\right) + \log{ \frac{p^2(1-q)^2}{4\eta(1-\eta)q^2(1-p)^2}}\left(A+A^T\right) + 2\log{\frac{1-p}{1-q}} (J -I)\\
             & \triangleq w_i \iu \left(A - A^T\right) + w_r \left(A+A^T\right) + w_c (J-I). \nonumber
        \end{align}
\end{restatable}

\subsection{A flow optimization interpretation of MLE}
\label{sec:optimization}
We now present an alternative interpretation of \eqref{opt:Herm-MLE} through the lens of flow-based optimization problems. 
This perspective not only provides additional intuition behind the objective but also paves the way for model-free generalizations of our approach.  
To formalize this connection, we first define the following flow-based graph statistics.
\begin{definition}
Given two clusters $\ca,\cb$ in a directed graph , we use $\TF(\ca,\cb)$ to denote the \emph{total flow} between $\ca$ and $\cb$ and $\NF(\ca,\cb)$ to denote the \emph{net flow} from $\ca$ to $\cb$, where
\begin{align*}
    \TF(\ca,\cb)=\sum_{\substack{ u \in \ca v \in \cb}} (A_{uv} + A_{vu}), \quad \NF (\ca,\cb) = \sum_{\substack{ u \in \ca v \in \cb}} (A_{uv} - A_{vu}).
\end{align*}
\end{definition}
Recall from Theorem~\ref{thm:MLE} that the MLE reduces to maximizing the following quadratic form
\begin{align}
\label{eq:obj_terms}
     \bold{x}^*\Hmle \bold{x} = w_r \bold{x}^*(A+A^T)\bold{x} + \iu w_i \bold{x}^*(A-A^T)\bold{x} + w_c \bold{x}^*(J-I)\bold{x}, 
\end{align}
where the vector $\bold{x} \in \{1,\iu\}^N$ indicates the cluster assignment for each vertex. 
Each term in this expression captures a different structural property of the graph.
\begin{itemize}[leftmargin = 12pt]
    \item \textbf{Edge density:} The first term, $\bold{x}^*(A+A^T)\bold{x}$, computes twice the number of edges within clusters. Equivalently, if we denote  $|\mathcal{E}|$ the total number of edges in the graph, we can write 
    \begin{align*}
        \bold{x}^*(A+A^T)\bold{x} = 2|\mathcal{E}|-2\TF(\ca,\cb).
    \end{align*}
    \item \textbf{Directional flow:} The second term ${\iu}\bold{x}^* (A-A^T) \bold{x}$ captures the net flow from $\ca$ to $\cb$, and simplifies to  
    \begin{align*}
        {\iu}\bold{x}^* (A-A^T) \bold{x} = 2 \NF(\ca,\cb)
    \end{align*}
    \item \textbf{Cluster size:} The last term, $\bold{x}^*(J-I)\bold{x}$, computes $|\ca|^2 + |\cb|^2$, which is regularization term that penalize{s} imbalanced cluster sizes.
\end{itemize}
Together, these three components define a regularized objective that balances flow-based graph statistics between $\ca$ and $\cb$.
Specifically, the MLE objective~\eqref{eq:obj_terms} corresponds to
\begin{align*}
    -w_r \TF (\ca,\cb) + w_i \NF(\ca,\cb) + 1/2 \cdot w_c (|\ca|^2 + |\cb|^2),
\end{align*}
where the weights $w_r$,  $w_i$ reflect the relative importance of edge density and directionality, respectively. The last term penalizes imbalanced partitions, but its impact is often minor as $w_c = \log(1-p) -\log(1-q)  =  O(p-q)$, which is typically much smaller than the other weights. 

The directionality weigh  $w_i = \log{\left(\frac{1-\eta}{\eta}\right)}$ is always non-negative for $\eta\in(0,0.5]$, and decreases monotonically as direction noise $\eta$ increases, This reflects a reduced emphasis on directionality when edge directions information becomes less reliable. The edge density weight weight $w_r$ captures both the uncertainty in edge directions and the disparity between intra- and inter-cluster connectivity. It tasks the form
$w_r = \log{\left( \frac{1}{4\eta(1-\eta)}\right)} + 2\log{ \left(\frac{p(1-q)}{q(1-p)}\right)}$. This term increases when $\eta$ decreases (i.e., when edge directions are more informative), and when the difference between $p$ and $q$ becomes more pronounced.

\textbf{Relation to other Hermitian formulations.} The statistical inference formulation provides a model-based justification for the clustering criterion and the corresponding matrix representation.  Our {newly derived}  Hermitian matrix arises naturally by optimizing sufficient statistics (total flow and net flow) weighted by their relative informativeness.  Such a framework enables a principled approach to designing new Hermitian data matrices for clustering, even without assuming a full generative model, but only using coarse prior knowledge about community structure.  
For example, when edge density is uninformative ($p \approx q$) and directionality is highly information ($\eta \approx 0$), the Hermitian matrix $H = A + A^T + \iu (A - A^T)$ serves as a practical approximation to the MLE-derived formulation in \eqref{eq:H-mle}. 
Prior work by \citet{cucuringu-2020-hermitian} proposed the matrix $H = i(A - A^\top)$, which corresponds to a special case of our formulation.   Their approach focuses solely on net flow and can be interpreted as MLE restricted to the cross-community block of the DSBM.

%% file: sections/analysis.tex
\subsection{Spectral relaxation and clustering algorithm}
Exactly solving the combinatorial optimization problem \eqref{opt:Herm-MLE} is NP-hard.
For computational efficiency, we consider the following spectral relaxation.

\textbf{Relaxation step.} We drop integer constraints in $\mathbf{x}$ and relax it to the continuous complex domain. This leads to the following continuous optimization problem 
    \begin{align*}
    \max \quad & \bold{x}^* \Hmle \bold{x} \tag{SC-MLE}\label{opt:MLEsc}\\
    s.t. \quad & {\bold{x} \in \C^{N}}, \|\bold{x}\|_2^2 =N.
    \nonumber
\end{align*}
This continuous problem \eqref{opt:MLEsc} is analytically solvable, and its maximum is attained when $\bold{x}$ is a rescaled (by $\sqrt{N}$) version of the leading eigenvector of $\Hmle$.

\textbf{Projection step.}
Note that the solution to the spectral relaxation, the top eigenvector of $\Hmle$, is not unique as its multiplication with $e^{\iu \theta}$ is also a top eigenvector for any $\theta \in [0,2\pi]$. 
To recover discrete labels, we apply the $k$-means algorithm (with $k=2$) to the complex-valued embeddings, treating each complex-valued vertex embedding as a point in $\mathbb{R}^2$.  This projection step yields rotation-invariant cluster assignments from the relaxed solution.

\subsection{Error analysis}
The intuition behind the success of spectral relaxation clustering is that the expected Hermitian matrix $\E[\Hmle]$ has community-dependent structure, and its leading eigenvector $\bold{v} = \bold{v}_1(\E[\Hmle])$ exactly encodes the true community labels through two distinct values. In practice, we observe the empirical matrix $\Hmle$, which serves as a perturbed version of {the unknown} $\E[\Hmle]$. Classical matrix perturbation theory ensures that if $\Hmle$ is sufficiently close to $\E[\Hmle]$, the leading eigenvector $\topv = \bold{v}_1(\Hmle)$ remains informative {and is sufficiently aligned with the top eigenvector of  $\E[\Hmle]$, thus enabling} accurate recovery of the community structure.
We outline the proof strategy below with full details deferred to Appendix~\ref{appx_perturbation}:
\begin{itemize}[leftmargin = 12pt]
    \item \textbf{Ideal case:}
    We first characterize the top eigenvector of $\E[\Hmle]$ in Lemma~\ref{lemma:EH_eig}, which exactly recovers the true community assignment.  
    \item  \textbf{Eigenspace perturbation:} 
    We apply two classical results, the Davis-Kahan perturbation bound and Weyl's inequality, to upper bound the eigenspace {misalignment} $ \|{\topvE}{\topvE}^* - \topv \topv^*\|_F$.
    Lemma~\ref{lemma:top-eig_pert} shows that this misalignment is bounded in terms of the perturbation norm  $\|\Hmle - \E[\Hmle]\|$ and the eigengap $\lambda_1(E[\Hmle]) - \lambda_2(\E[\Hmle])$.  
    In contrast to prior work (e.g., \cite{cucuringu-2020-hermitian}), our result holds without any technical assumptions on the eigengap, due to a variant of the Davis–Kahan theorem\citep{vu-2013-minimax, yu-2015-useful}.
    \item \textbf{Matrix concentration:}
    We apply the Matrix-Bernstein inequality from random matrix theory, and provide in Lemma~\ref{lemma:bound_perturb} a high-probability upper bound on the perturbation norm $\| \Hmle - \E[\Hmle]\|$.
    \item \textbf{Final error bound:}
    Combining the above results with a standard error bound for the $k$-means projection step (Lemma~\ref{lemma:km_error}), we establish a clustering error bound in Theorem~\ref{thm:errorHermSC} (see Appendix~\ref{appx:proof_thm2} for proof).
\end{itemize}


Let $\sigma$ denote the true community assignment and $\hat{\sigma}_{\text{SC-MLE}}$ be the $(1+\epsilon)$-approximate solution obtained by applying $k$-means{++} to the spectral relaxation of the MLE problem. The misclustering error is measured by the Hamming distance between the ground truth and estimated solution \begin{align*}
    l(\sigma, \hat{\sigma}) = \sum_{u\in \mathcal{V}} \mathbb{1}\{\sigma_u \neq \hat{\sigma}_u\}.
\end{align*}

For the main result of the misclustering error bound, we introduce the following  standard assumption  $N \pmax = \Omega\left({\log{N}}\right)$
with $\pmax = \max\{p,q\}$. This technical assumption is to allow good concentration properties of the random graph (see Lemma~\ref{lemma:bernstein}), which is common in the realm of spectral methods equipped with theoretical guarantees.
\begin{restatable}[Error bound of \ref{opt:MLEsc}]{theorem}{ThmErrorSC}
\label{thm:errorHermSC}
    For graphs generated from the  DSBM $(n_1,n_2,p,q,\eta)$ with $N \pmax = \Omega\left({\log{N}}\right)$,  there exists  $C = \Theta\left(\sqrt{w_r^2 + w_i^2}\right)$ (see \eqref{eq:C}) and an absolute constant $\epsilon_0$, such that with probability at least $1-N^{-\epsilon_0}$, the error rate 
    \begin{align}
    \label{eq:Thm-error}
        \frac{l(\sigma, \hat{\sigma}_{\text{SC-MLE}})}{N}\leq  \frac{64(2+\epsilon)C^2{\pmax \log{N}} }{ d^2\Delta^2 }.
    \end{align}
    Here $\Delta$ and $d$ depends only on the population matrix $\E[\Hmle]$, where $\Delta$ lower bounds the eigengap $\lambda_1( \E[\Hmle]) - \lambda_2(\E[\Hmle])$ with expression given in \eqref{eq:eigengap};
    $d$ denotes the distance between the two cluster centroids of the population version $\E[\Hmle]$ with expression provided in \eqref{eq:centroids}.
\end{restatable}

To better interpret the connection between the error bound in Theorem~\ref{thm:errorHermSC} and the noise level in DSBM,  we consider a simplified case that admits an explicit analytical form for the error bound. Specifically, in Corollary~\ref{coro:eta}, we specialize the error bound for a symmetric DSBM with homogeneous edge probabilities, a setting that lies below the detection threshold for undirected graphs.  The proof of Corollary~\ref{coro:eta} can be found in Appendix~\ref{appx_perturbation}.  
\begin{restatable}{corollary}{Coroeta}   
\label{coro:eta}   
    Consider directed graphs generated from the DSBM $(N/2,N/2,p,p,\eta)$ under the assumption $Np = \Omega(\log{N})$. As $N \rightarrow \infty$, the misclustering error of  spectral clustering is such that  
    \begin{align}
    \label{eq:coro1_eta_bound}
        \frac{l(\sigma,\hat{\sigma}_{\text{SC-MLE}})}{N}  = \Theta\left(\frac{\log{N}}{NpL^2} \right),
    \end{align}
 where $L = L(\eta)$ is a continuous, monotonically decreasing function of the edge directionality parameter $\eta$, with $L = 0$ when $\eta = 0.5$.  The explicit form of $L(\eta)$ is provided in \eqref{eq:L_eta}, and its behavior is illustrated in Figure~\ref{fig:LEta} in Appendix~\ref{appx_perturbation}.
\end{restatable}

From \eqref{eq:coro1_eta_bound}, we observe that the upper bound on the misclustering error is inversely proportional to the average degree $Np$, which aligns with the intuition that a larger average degree provides more edge observations, rendering the generated graph more informative and leading to a smaller {mis}clustering error. Furthermore,  as the edge direction noise $\eta$ increases, the function $L(\eta)$ decreases, resulting in a higher misclustering error. In particular, as $\eta\rightarrow 0.5$, we have $L \rightarrow 0$, indicating a sharp increase in errors as directional information gradually disappears.  This behavior is consistent with the intuition that {a higher noise level in the edge orientations} weakens the structural information to identify communities, {rendering the recovery task} more difficult. 
Moreover, this homogeneous edge density setting highlights the importance of exploiting edge orientation in directed graph clustering: if edge directions are ignored, the model becomes statistically indistinguishable from an Erd\H{o}s--R\'enyi graph, making community detection impossible.

%% file: sections/experiment.tex
\subsection{Algorithm: \texttt{LE-SC}}
Directly applying the spectral relaxation of Herm-MLE, as explained in the previous section, might be impractical as it requires knowledge of the model parameters $p,q$, and $\eta$. To circumvent this, we adopt an iterative parameter learning approach based on pseudo maximum likelihood estimation \citep{gong-1981-pseudo}, where the nuisance parameters $p,q$, and $\eta$ are re-estimated using the method of moments in each iteration. This leads to a novel spectral clustering algorithm for directed graphs, which we summarize in Algorithm~\ref{alg:LE-SC}. 
\begin{algorithm}
\SetKwInOut{Input}{Input}\SetKwInOut{Output}{Output}
\caption{Likelihood Estimation Spectral Clustering (\texttt{LE-SC})}
\label{alg:LE-SC}
\Input{Directed graph $G(\mathcal{V},\mathcal{E})$, number of maximum iteration $T$}
\Output{Community labels $\hat{\sigma}$}
\textbf{Initialize:} Randomly set initial values for DSBM parameters $p$, $q$, and $\eta$\;
\For{$t = 1$ \KwTo $T$}{
Compute the Hermitian matrix $\Hmle$ using \eqref{eq:H-mle}\;
Compute the top eigenvector $\topv$ of $\Hmle$\;
Apply $k$-means to the embedding $[\Re(\topv), \Im(\topv)]$ to partition into two clusters: $\ca$ and $\cb$\;
    Update DSBM parameters based on current clustering:\;
    \Indp
    $p \gets \frac{|\mathcal{E}| - \TF(\ca, \cb)}{\binom{|\ca|}{2} + \binom{|\cb|}{2}}$\hfill // within-community edge density\;
    $q \gets \frac{\TF(\ca, \cb)}{|\ca| \cdot |\cb|}$\hfill // between-community edge density\;
    $\eta \gets \min \left\{ \frac{|\ca \rightarrow \cb|}{\TF(\ca, \cb)},\ \frac{|\cb \rightarrow \ca|}{\TF(\ca, \cb)} \right\}$\hfill // directionality asymmetry\;
    \Indm
}
\Return{Final community labels $\hat{\sigma}$.}
\end{algorithm}

Here, \texttt{LE-SC} partitions a directed graph into two clusters. To obtain more than two clusters, we recursively apply the bipartition procedure, each time splitting the largest remaining cluster until the desired number of clusters is reached. The full multi-cluster extension is summarized as Algorithm~\ref{alg:LE-SC-k} in Appendix~\ref{appx:multi-cluster}.
We refer to both the two-cluster and multi-cluster versions of our method as \texttt{LE-SC}.

\textbf{Convergence.} While there is no known theoretical guaranteens on the convergence of iterative pseudo maximum likelihood estimation for stochastic block models, our extensive experiments on DSBMs suggest that the algorithm is robust to different initializations and typically converges to the true parameters within 10 iterations ({an observation in line with what has been reported in} \citep{newman-2016-equivalence}). For a practical implementation, we recommend assigning the model parameters based on an initial clustering output from existing algorithms, in order to facilitate the convergence. Detailed experimental studies on convergence are provided in Appendix~\ref{appx:convergence}, and comparisons between \texttt{LE-SC} and spectral clustering with true model parameters are presented in Figure~\ref{fig:DSBM_k2}.

\textbf{Complexity analysis.} The algorithm \texttt{LE-SC} involves several iterations of computing the leading eigenvector of the Hermitian matrix \eqref{eq:H-mle}, where each eigenvector computation requires $\mathcal{O}(|\mathcal{E}|)$ operations. Note that although the Hermitian matrix includes a dense all-ones component, the eigenvector computation can still take advantage of the graph sparsity and be performed in $\mathcal{O}(|\mathcal{E}|)$ time by decomposing the matrix operations (see Appendix~\ref{appx:complexity} for a detailed discussion).
{In the setting of multiple clusters}, we apply \texttt{LE-SC} to recursively bi-partition the largest remaining cluster at each step. Therefore, clustering into $k$ groups requires an overall computational complexity of $\mathcal{O}(k|\mathcal{E}|)$.  All experiments were conducted on a MacBook Pro with an Apple M2 chip, 24 GB of RAM. The algorithm \texttt{LE-SC} completes in several seconds on input graphs with thousands of vertices.

\subsection{Experiments on synthetic DSBM graphs}
\textbf{Data generation overview.} We conduct experiments on directed graphs sampled from the DSBM ensemble, with each community having a fixed size of $1000$, and varying the number of communities as well as model parameters $p,q$, and $\eta$. Since spectral clustering typically performs well for dense graphs, we specifically focus on the more challenging sparse regime, where the edge probabilities $p$ and $q$ are slightly above $\log n/n$, the connectivity threshold of random graphs.

\textbf{Baselines.} We compare against several baseline spectral clustering algorithms: (a) Hermitian-based methods: \texttt{Herm} \citep{cucuringu-2020-hermitian} and \texttt{SimpHerm} \citep{laenen-2020-higher}; (b) SVD based methods: \texttt{DI-SIM(L)} and \texttt{DI-SIM(R)} from \citep{rohe-2016-coclustering}, and \texttt{D-SCORE} \citep{wang-2020-dscore};
and (c) symmetrization-based methods: naive symmetrization \texttt{Sym} (using $A+A^T$), bibliometric symmetrization \texttt{Bib-Sym} 
 (using $AA^T + A^TA$) \citep{satuluri-2011-symmetrizations}.

\textbf{Evaluation metric.}
We assess the clustering performance using the Adjusted Rand Index (ARI) \citep{gates-2017-ARI_NMI}, which quantifies the similarity between the clustering outcomes and ground-truth labels.  The ARI ranges from $-1$ to $1$, with higher values indicating better clustering performance: ARI value of $1$  indicates perfect recovery and $0$ implies that the recovery is almost like a random guess. In each synthetic experiment, we independently sample 10 directed graphs with a fixed parameter set, and report the averaged ARIs over these graph samples. 

\subsubsection{Two-community DSBMs.} 

We perform experiments on two-community DSBMs with varying model parameters and summarize {the results} in Figure~\ref{fig:DSBM_k2} the ARIs. Overall, our proposed algorithm $\texttt{LE-SC}$, consistently outperforms existing methods by a significant margin. In particular, we compare  $\texttt{LE-SC}$ with its oracle version, where the true parameters $p$, $q$, and $\eta$ are provided, bypassing the iterative pseudo-maximum likelihood estimation. The nearly identical performance between $\texttt{LE-SC}$ and $\texttt{LE-SC(oracle)}$ suggests that the theoretical guarantee established for the MLE in Theorem~\ref{thm:errorHermSC} may extend to the practical pseudo-likelihood estimation algorithm $\texttt{LE-SC}$.

We evaluate both homogeneous ($p = q$) and inhomogeneous ($p \neq q$) DSBM settings.
In the homogeneous case, community recovery can only rely on the information attached to the edge directions. In contrast, when $p \neq q$, clusters are informed by both the direction and the density of directed edges. Our method is self-adaptive, effectively leveraging both types of information. For example, when edge direction is highly informative (small $\eta$), \texttt{LE-SC} significantly outperforms symmetrization-based baselines by exploiting directionality. 
As the directional information diminishes (i.e., as $\eta$ approaches $0.5$),  the performance of \texttt{LE-SC} gradually decreases to that of the symmetrized baseline, demonstrating its adaptiveness to varying  information sources.

\begin{figure}[htp!]
    \centering
    \begin{subfigure}{0.22\textwidth}
        \includegraphics[width=\textwidth]{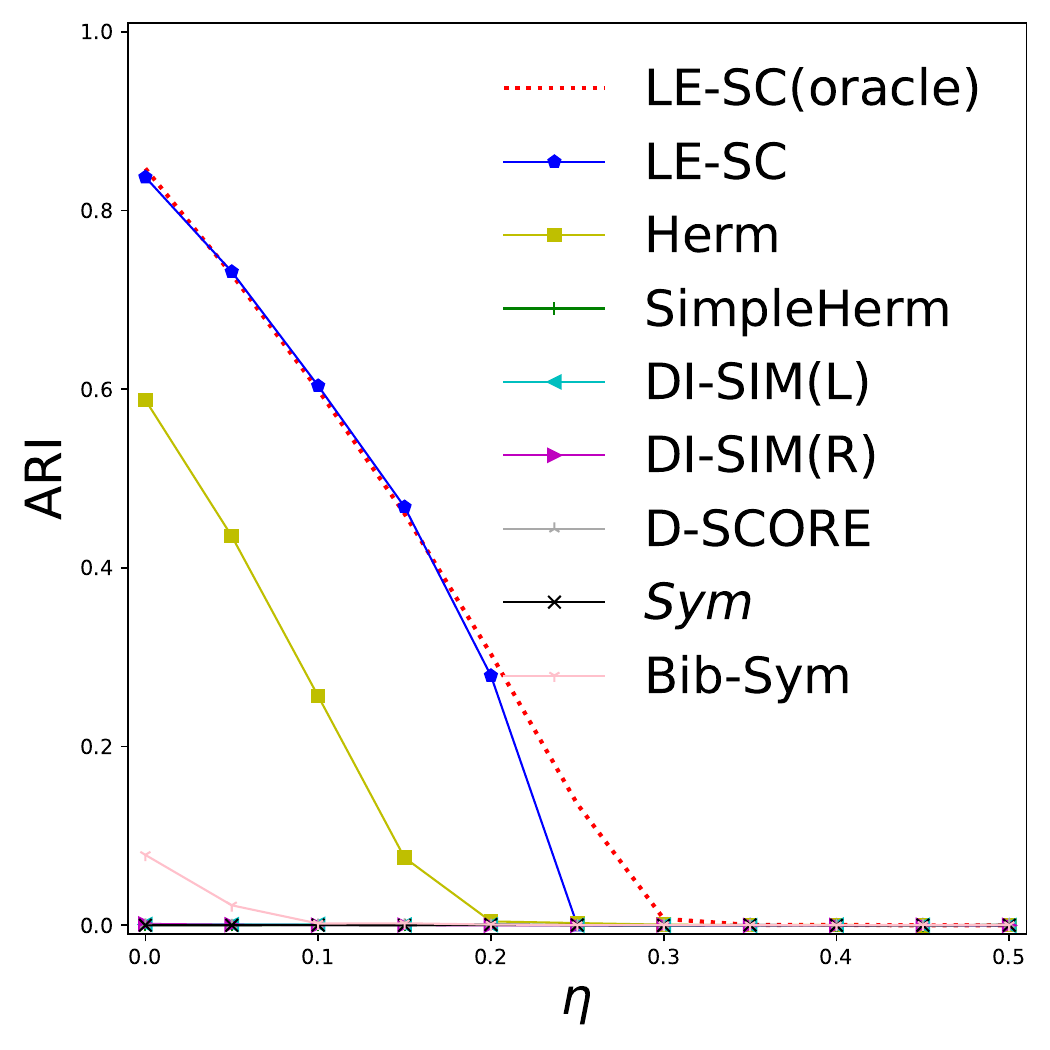}
        \caption{$p = q = 0.5\%$}
    \end{subfigure}
    \begin{subfigure}{0.22\textwidth}
        \includegraphics[width=\textwidth]{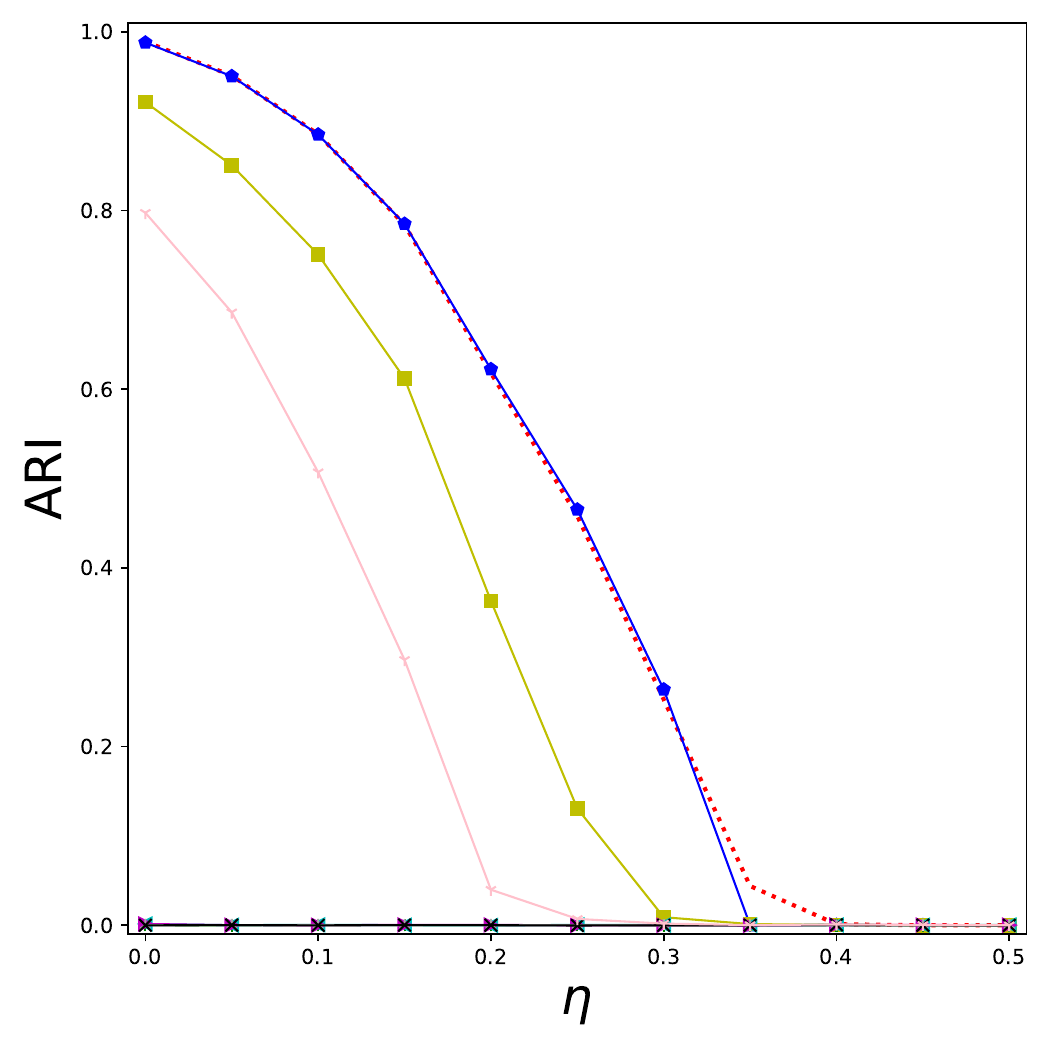}
        \caption{$p = q = 1\%$}
    \end{subfigure}
    \begin{subfigure}{0.22\textwidth}
        \includegraphics[width=\textwidth]{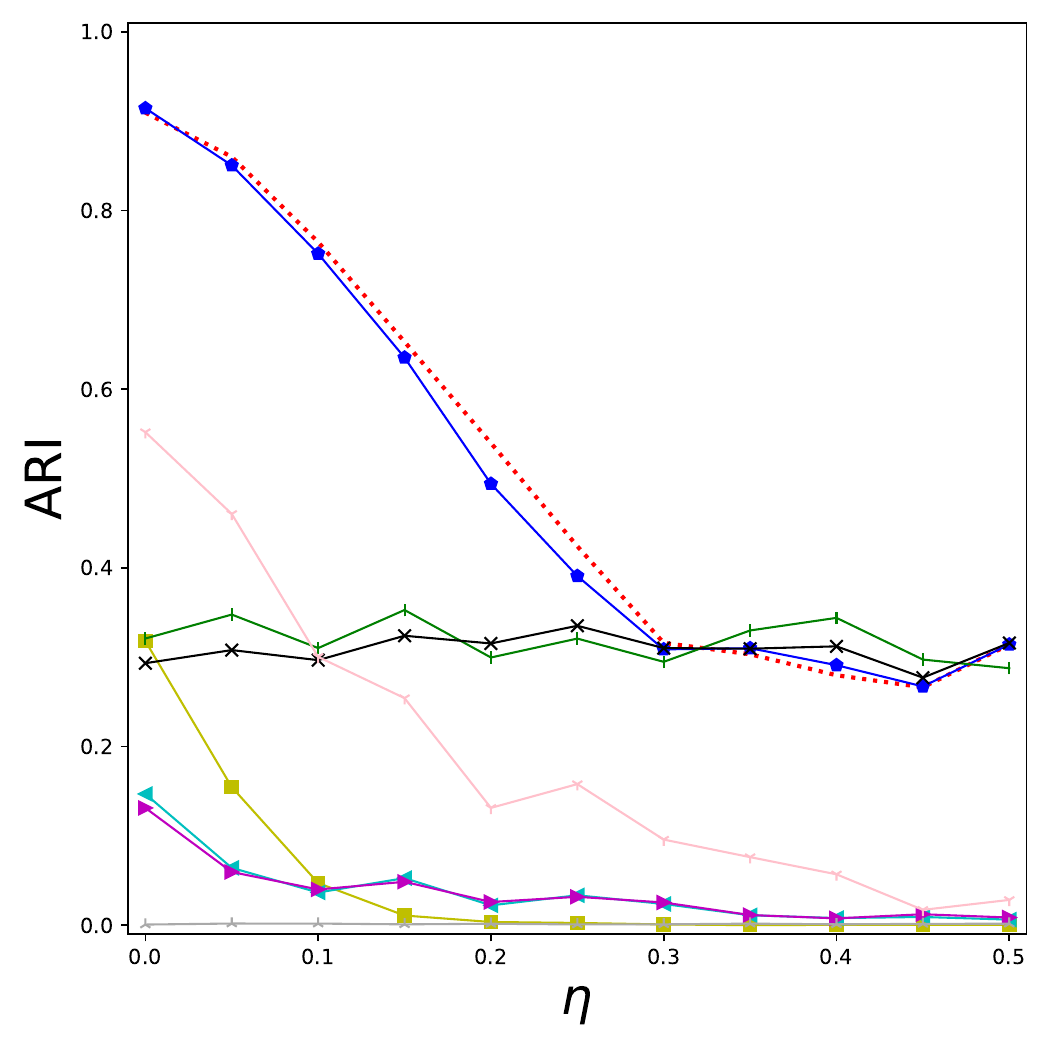}
        \caption{$p = 1\%$, $ q = 0.5\%$}
    \end{subfigure}
    \begin{subfigure}{0.22\textwidth}
        \includegraphics[width=\textwidth]{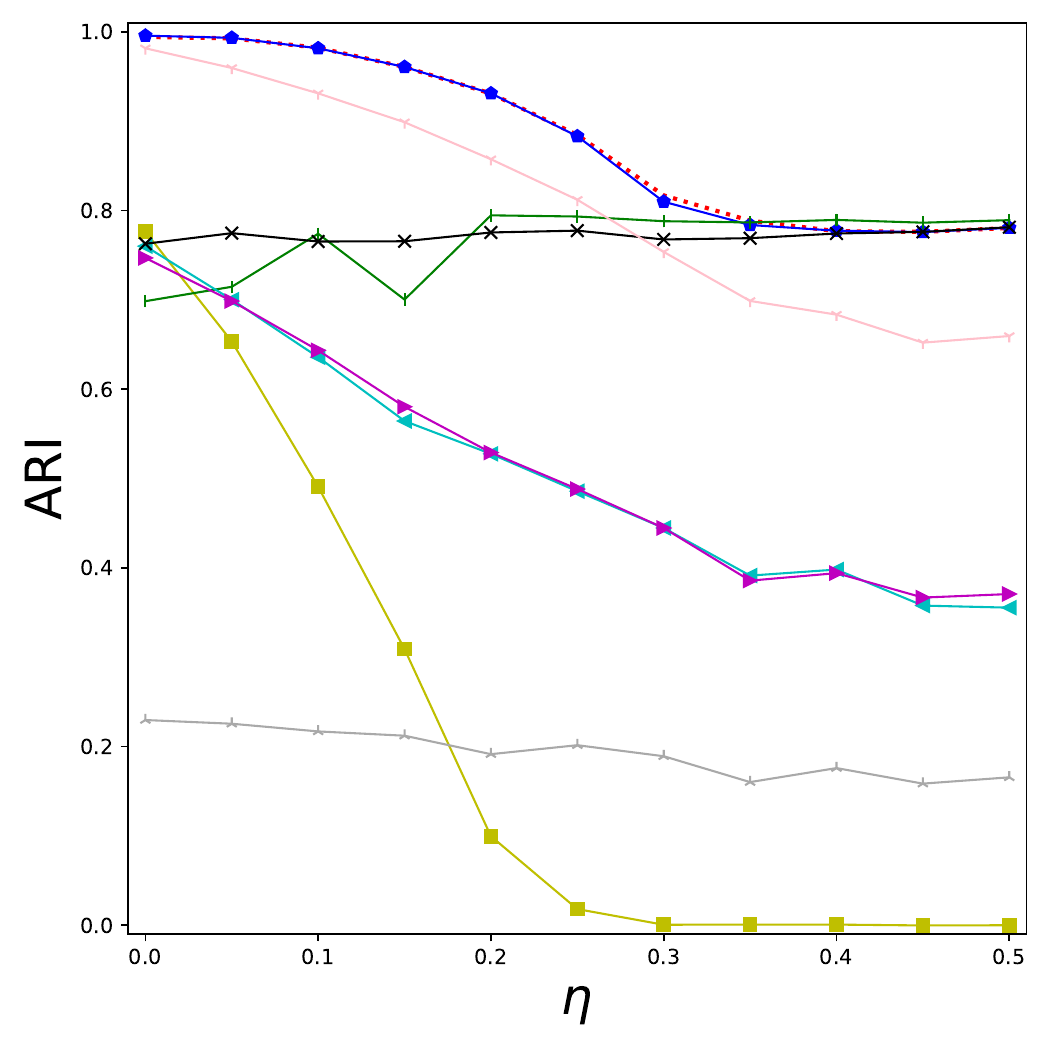}
        \caption{$p = 2\%$, $ q = 1\%$}
    \end{subfigure}
    \caption{Experiments on two-community DSBM with varying model parameters.}
    \label{fig:DSBM_k2}
\end{figure}
\subsubsection{Multi-community DSBMs}
For DSBMs with multiple communities, the interaction between communities form{s a higher-order} pattern, which we represent using a meta-graph: vertices are communities and directed edges encode the orientation parameter $1-\eta$ between the two community pairs, while the absence of an edge implies randomly assigned edge directions. 

While exhaustive testing on all possible multi-community DSBMs is infeasible, we present experimental results on DSBMs with meta-graph structures illustrated in Figure~\ref{fig:k3} and Figure~\ref{fig:k4}. 
Overall, our method achieves competitive performance compared to baselines across these settings.  We further conduct extensive experiments with different higher-order meta-graph structures and observe that the relative performance of various algorithms depends significantly on the structure of the meta-graph topology.  Although there is no known theoretical analysis on this, we provide more empirical studies and discussions in Appendix~\ref{appx:k-DSBM} to better understand the strengths and limitations of our method in multi-community settings.

\begin{minipage}{0.28\textwidth}
    \begin{figure}[H]
    \centering
    \begin{subfigure}{0.4\textwidth}
        \includegraphics[width=1.2\textwidth]{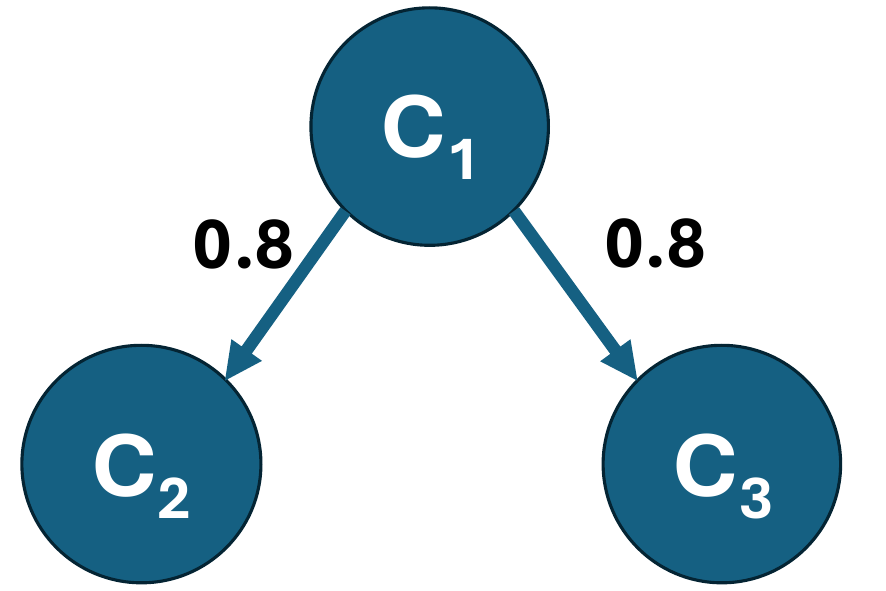}
        \vspace{-12pt}
        \caption{}\label{fig:k3}
    \end{subfigure}
    \hfill
    \begin{subfigure}{0.4\textwidth}
        \includegraphics[width=1\textwidth]{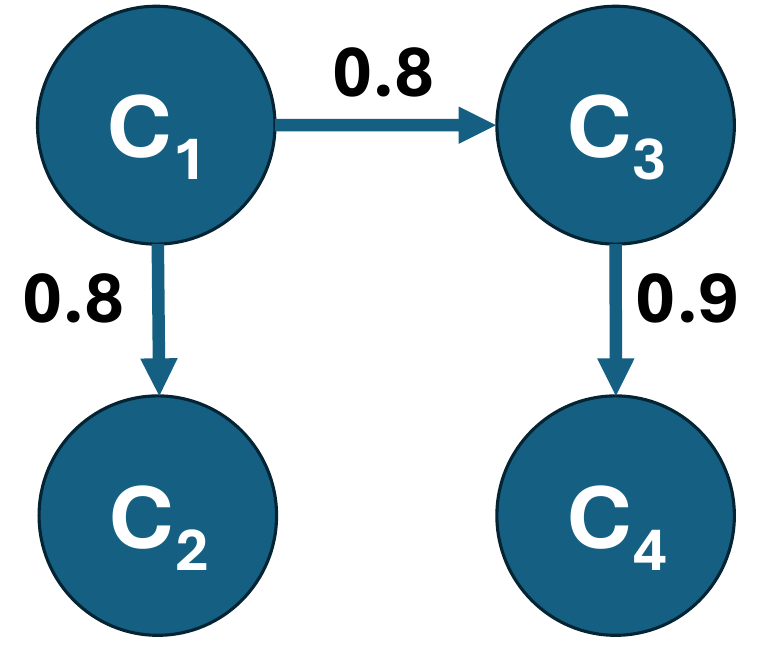}
        \vspace{-12pt}
        \caption{}\label{fig:k4}
    \end{subfigure}
    \caption{Meta-graphs.} \label{fig:meta}
    \end{figure}
    \end{minipage}
    \hfill
    \begin{minipage}{0.7\textwidth}
    \centering
    {\setlength{\tabcolsep}{3pt}
    {\small
    \begin{tabular}{lccccccccc}
    \toprule
 Group& \texttt{Sym} & \texttt{Bib-Sym}  & \texttt{DI-SIM} & \texttt{D-SCORE} & \texttt{Herm} & \texttt{SimpHerm} & \texttt{LE-SC} \\
    \midrule
    (a-1) & 0.00 & 0.16 & 0.00 & 0.00& \textbf{ \blue{0.31}} & 0.00 & \textbf{\red{0.41}} \\
    (a-2) & \textbf{\blue{0.66}} & {{0.65}} & 0.28 & 0.02 & {0.19} & 0.00 & \textbf{\red{0.83}} \\
    (b-1) & 0.00 & 0.21 & 0.01& 0.00& \textbf{ \red{0.45}} & 0.10 & \textbf{\blue{0.29}} \\
    (b-2) & 0.58 &  \textbf{\red{0.72 }}&  0.31 &0.01 & 0.32& 0.11& \textbf{\blue{0.59}}\\
\bottomrule
\end{tabular}
\captionof{table}{ARIs on clustering multi-community DSBMs. The prefix in Group denotes corresponding meta-graphs from Figure~\ref{fig:meta}, while suffix "1" denotes $p = q = 1\%$ and suffix "2" denotes $p = 2\%, q = 1\%$.}
\label{tab:DSBM_k}
}
}
\end{minipage}

\subsection{Larval Drosophila mushroom body connectome}
In neuron-neuron interaction networks, vertices represent neurons and directed edges correspond to synapses, where the edge orientation often reflects the functional roles of neuron types. 
We study the Larval Drosophila mushroom body connectome, a real-world neuron connectivity graph that contains three types of neurons: Kenyon Cells (KC), Output Neurons (MBON), and Projection Neurons (PN). According to neuroscience research \citep{eichler-2017-complete, priebe-2017-semiparametric}, these neuron types exhibit characteristic patterns of directed connectivity (Figure~\ref{fig:neuron_true}), and can be modeled using a meta-graph shown in Figure~\ref{fig:neuron_c}.
We apply spectral clustering algorithms to recover the neuron types based solely on the observed directed graph. The  adjacency matrices, 
{ordered by the inferred cluster membership}, 
along with the corresponding ARI scores, are shown in Figure~\ref{fig:adj_neoron}.
Among the eight spectral algorithms evaluated, our method \texttt{LE-SC} achieves the highest recovery accuracy, and its visualization reveals a clear directional structure between the (1,3) and (2,3) blocks, corresponding to strong directed connections from KC to MBON and PN to KC. 

\begin{figure}[h]
    \centering
    \begin{minipage}{0.3\textwidth}
        \centering
        \begin{subfigure}{\textwidth}
        \centering
            \includegraphics[width=0.5\linewidth]{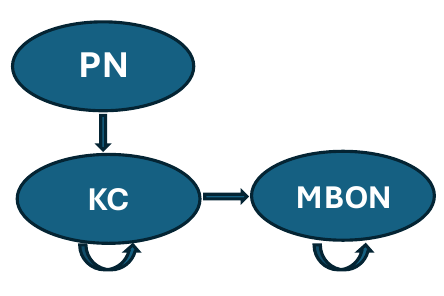}
        \caption{Relation between neuron types}
        \label{fig:neuron_c}
        \end{subfigure}
        \begin{subfigure}{\textwidth}
        \centering
            \includegraphics[width=0.5\linewidth]{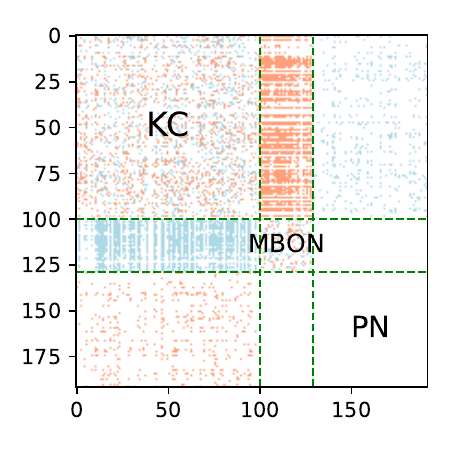}
        \caption{Gound truth}
        \label{fig:neuron_true}
        \end{subfigure}
    \end{minipage}
    \begin{minipage}{0.68\textwidth}
    \centering
\includegraphics[width=0.7\linewidth]{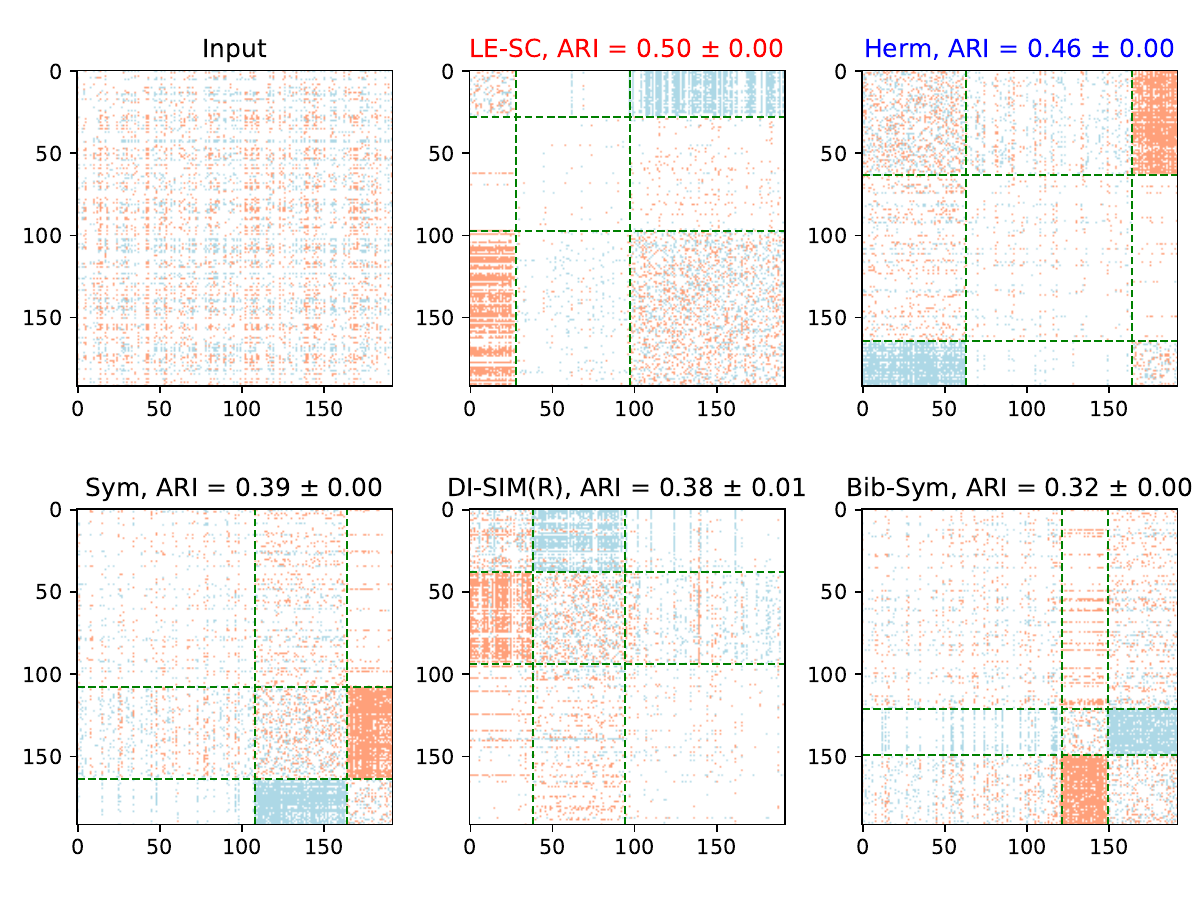}       
\end{minipage}
\caption{Visualization of graph adjacency {matrices} after clustering (results from the best {five} algorithms). A red pixel indicates an outgoing edge from the vertex indexed by the row to the vertex indexed by the column.}
\vspace{-12pt}
\label{fig:adj_neoron}
\end{figure}
\subsection{US migration network}
We consider migration patterns in the United States using data from the 2000 U.S. Census, which documents county-to-county migration between 1995 and 2000 \citep{census2002, cucuringu-2020-hermitian}.
We include 3074 mainland U.S. counties, and represent migration flows between counties using a directed and weighted graph. In this graph, each edge weight corresponds to the number of individuals migrating from one county to another.
To avoid a biased result dominated by extremely high degree vertices, we normalize the directed graph and use $D^{-1/2}A D^{-1/2}$, with the degree matrix $D$ accounting for both incoming and outgoing edges.
We then apply spectral clustering methods to partition the graph into $k=3$ clusters and visualize the outcomes in Figure~\ref{fig:migration}. Additional details and results for $k =2, 5$ and $10$ are provided in Appendix~\ref{sec:appx_US}. 

Notably, \texttt{LE-SC} {is the only algorithm, that}  
identifies a distinct cluster of economically significant metropolitan areas (the red cluster in Figure~\ref{fig:US_LESC_3}). 
This includes regions surrounding New York State, major cities in California, and urban hubs like Seattle, Dallas, Atlanta, Chicago, and Denver. 
The presence of these distant yet economically vibrant regions within the same cluster suggests that those migrations were likely to be driven more by economic opportunity than by geographic proximity, an insight not captured by existing baseline methods. 
The other two clusters identified by \texttt{LE-SC} (shown in blue and yellow) exhibit more geographically cohesive patterns.  For example, the blue cluster primarily contains counties from Minnesota and Wisconsin, likely reflecting regional migration trends.
Together, these results underscore the strength of \texttt{LE-SC} in detecting latent patterns in human mobility, revealing insights into migration dynamics beyond what existing baselines capture.

\begin{figure}[tph!]
    \centering
    \begin{subfigure}{0.24\textwidth}
        \includegraphics[width = \textwidth]{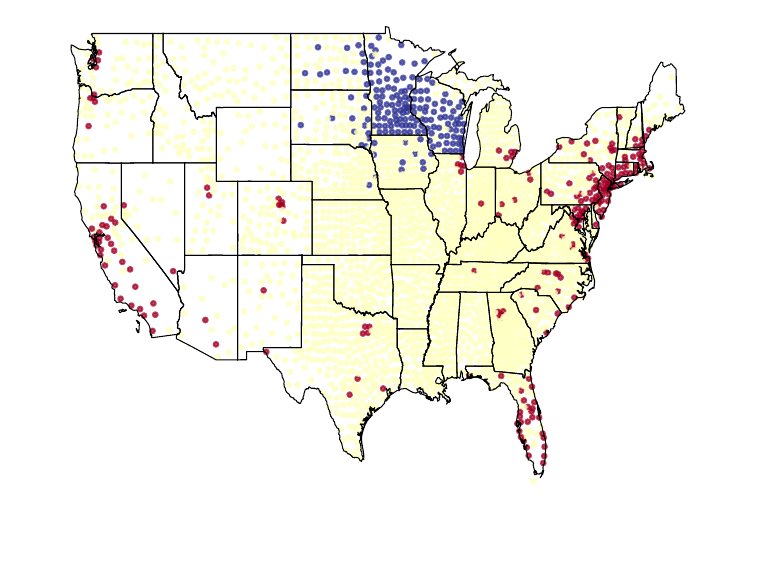}
        \vspace{-24pt}
        \caption{\texttt{LE-SC}}\label{fig:US_LESC_3}
    \end{subfigure}
    \begin{subfigure}{0.24\textwidth}
        \includegraphics[width = \textwidth]{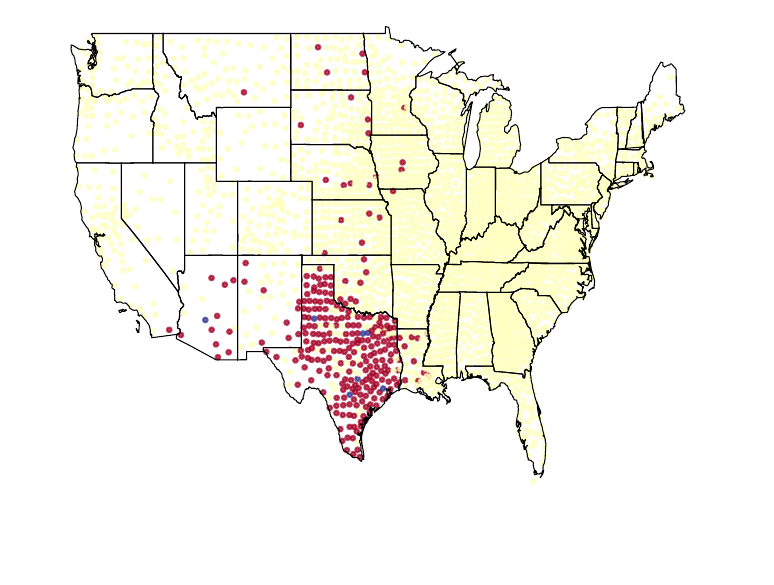}
        \vspace{-24pt}
        \caption{\texttt{Herm(RW)}}
    \end{subfigure}
    \begin{subfigure}{0.24\textwidth}
        \includegraphics[width = \textwidth]{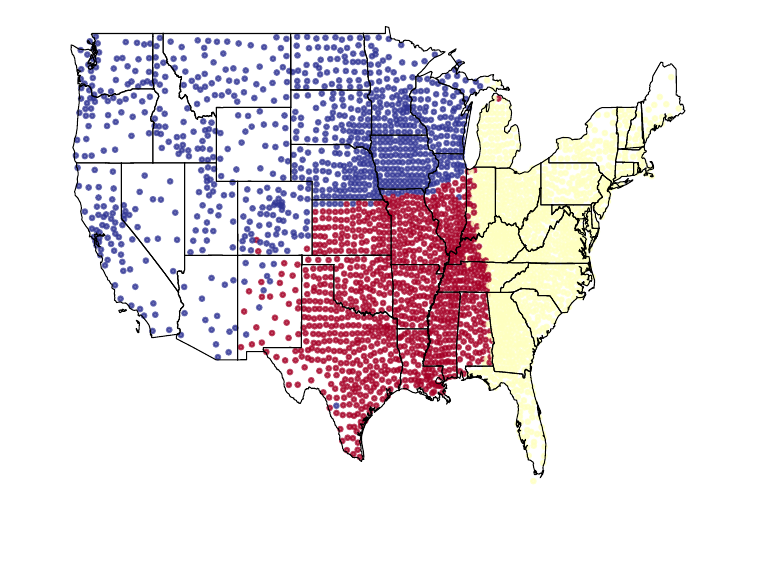}
        \vspace{-24pt}
        \caption{\texttt{DI-SIM(L)}}
    \end{subfigure}
    \begin{subfigure}{0.24\textwidth}
        \includegraphics[width = \textwidth]{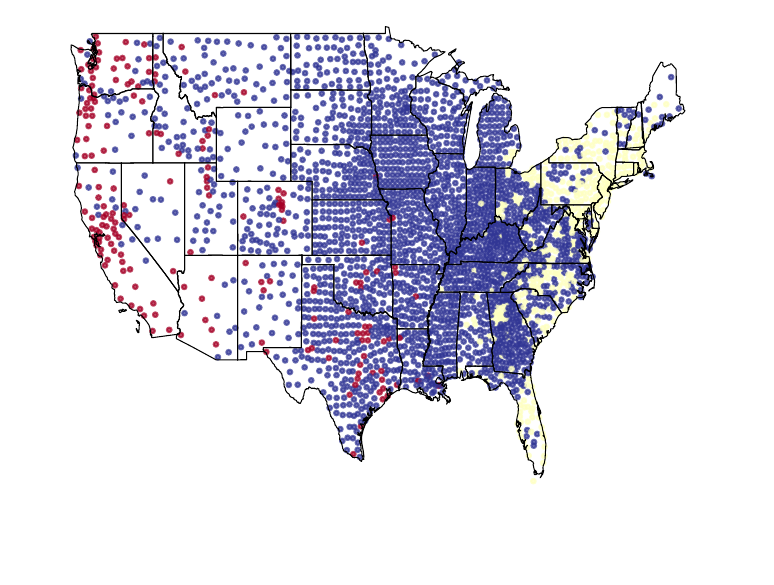}
        \vspace{-24pt}
        \caption{\texttt{Bib-Sym}}
    \end{subfigure}
    \begin{subfigure}{0.24\textwidth}
        \includegraphics[width = \textwidth]{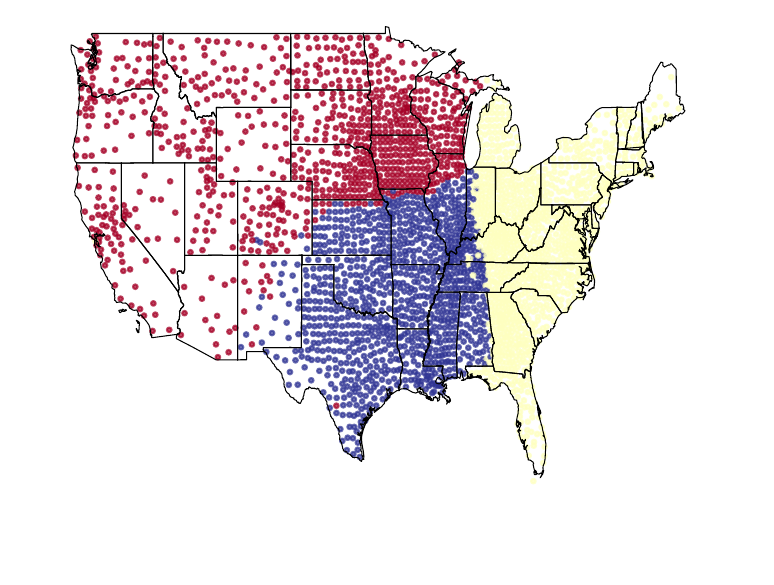}
        \vspace{-24pt}
        \caption{\texttt{DI-SIM(R)}}
    \end{subfigure}
    \begin{subfigure}{0.24\textwidth}
        \includegraphics[width = \textwidth]{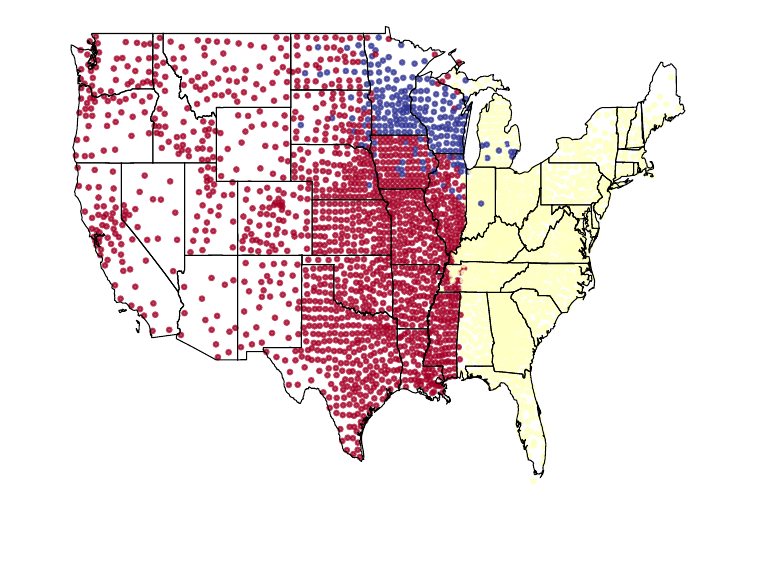}
        \vspace{-24pt}
        \caption{\texttt{Sym}}
    \end{subfigure}
    \begin{subfigure}{0.24\textwidth}
        \includegraphics[width = \textwidth]{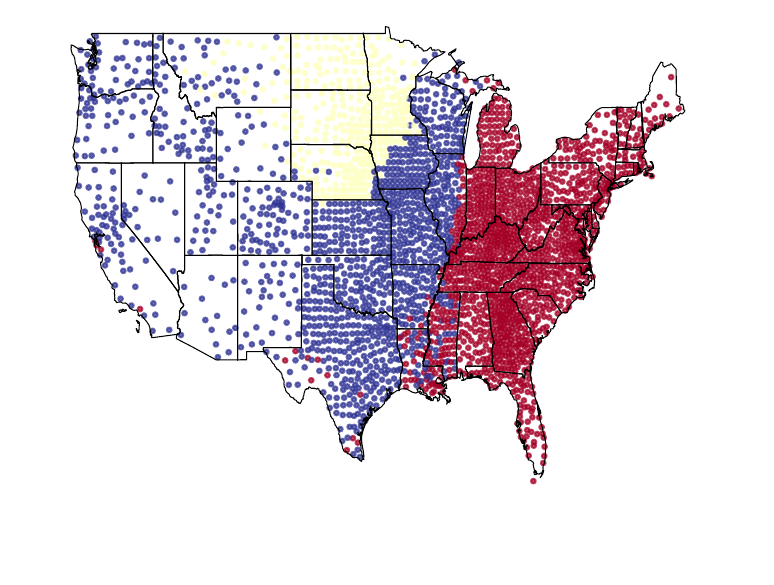}
        \vspace{-24pt}
        \caption{\texttt{SimpHerm}}
    \end{subfigure}
    \begin{subfigure}{0.24\textwidth}
        \includegraphics[width = \textwidth]{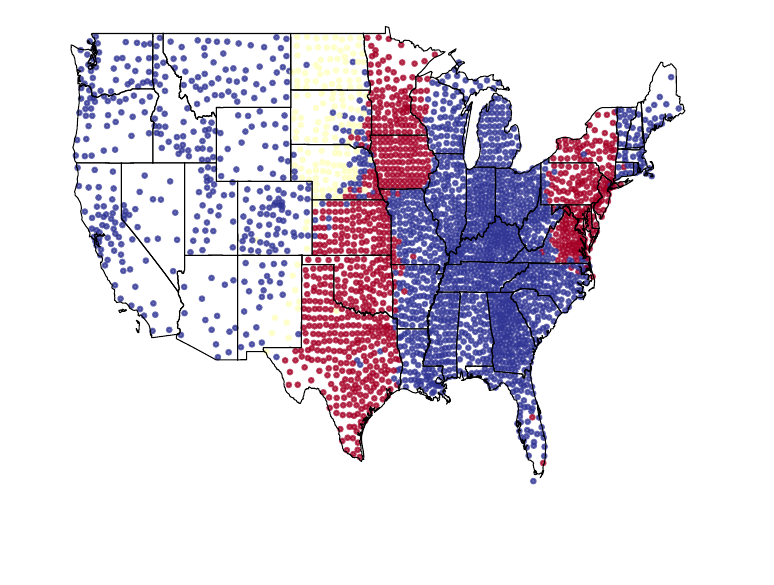}
        \vspace{-24pt}
        \caption{\texttt{D-Score}}
    \end{subfigure}
    \caption{US Counties clustered by migration data ($k=3$).}
    \vspace{-12pt}
    \label{fig:migration}
\end{figure}

%% file: sections/discussion.tex
This paper introduces a model-based methodology for clustering directed graphs, addressing the limitations of existing spectral methods that are largely heuristic, in the sense that they do not directly optimize a task-relevant objective function.
On the theoretical side, we derive a new asymptotic error bound for spectral methods under the DSBM, improving upon the existing bound from \cite{cucuringu-2020-hermitian} for the two-community case.
Our model-based framework enables a systematic understanding of how spectral clustering methods can recover communities by exploring structural heterogeneity induced by the community membership. On the practical side, we propose a fast and self-adaptive spectral algorithm that demonstrates superior clustering performance on a variety of synthetic and real-world graphs.

Despite these contributions, several limitations remain. For example, our model does not fully capture real-world complexities such as degree heterogeneity, which may limit the applicability of our algorithm in some settings. Extending the framework to more realistic models, such as degree-corrected DSBMs, is a promising direction for future work. 
Furthermore, our statistical inference formulation lends itself naturally to a Bayesian extension, allowing the incorporation of prior knowledge about community assignments when such side information is available.

%% file: sections/appx_notations.tex
\begin{longtable}{ c  l } 
\toprule
Notation & Definition\\ \midrule\midrule
$G(\mathcal{V},\mathcal{E})$ & Graph with vertex set $\mathcal{V}$ and edge set $\mathcal{E}$ \\
$u\rightsquigarrow v$ & There is an edge pointing from vertex $u$ to vertex $v$\\ 
$u \not\sim v$ & There is no edge between vertex $u$ and vertex $v$ \\
$A$ & Graph adjacency matrix, $A\in \{0,1\}^{n\times n}$ and $A_{uv} =1 $ iff $u\rightsquigarrow v$ \\ 
$\ca$ & Source community\\
$\cb$ & Target community\\
$\mathcal{V}$ & Set of all vertices, $\mathcal{V} = \ca \cup \cb$\\
$c(\cdot)$ & Community labelling function, $c(\cdot): \mathcal{V} \rightarrow \{1,2\}$ where $c(u) = 1$ iff $u \in \ca$ \\
$\Hmle$ & Hermitian matrix derived from MLE on DSBM (see \eqref{eq:H-mle})\\ 
$\sigma$ & A general community indicator vector, $\sigma_u = \sigma_v$ iff $c(u) = c(v)$\\
$\mathcal{L}(A;\sigma)$ & Log-likelihood function \\
$\TF(\ca,\cb)$ & Total flow between $\ca$ and $\cb$, $\displaystyle \TF(\ca,\cb)=\sum_{\substack{ u \in \ca v \in \cb}} (A_{uv} + A_{vu})$\\
$\NF (\ca,\cb)$ & {Net flow} from $\ca$ to $\cb$, $\displaystyle \NF (\ca,\cb) = \sum_{\substack{ u \in \ca v \in \cb}} (A_{uv} - A_{vu})$ \\ 
$|\ca \rightarrow \cb|$ & Number of edges from $\ca$ to $\cb$\\
$H^T$ & Transpose of $H$\\
$H^*$ & Conjugate transpose of $H$\\
$H_{j*}$ & The $j$-th row vector of $H$\\
$[H_1,H_2]$ & Concatenating columns of $H_1$ and $H_2$\\
$\|H\|$ & Spectral norm of $H$, $\|H\| = |\lambda_1(H)|$\\
$\|H\|_F$  &Frobenius norm of $H$, $\|H\|_F = \sqrt{\sum_{j} \lambda_j^2(H)}$\\
$\langle H_1,H_2 \rangle$ & Frobenius inner product, $\langle H_1,H_2 \rangle = \tr(H_1^* H_2)$  \\ 
$\diag(H)$ & Create a diagonal matrix by taking the main diagonal elements of $H$\\
$\Re(H)$ & Take the real part of the matrix $H$;\\
$\Im(H)$ & Take the imaginary part of the matrix $H$\\
$\bold{v_j}(H)$ & The $j$-th eigenvector of $H$\\
ARI & Adjusted Rand Index\\
$M \in \{0,1\}^{N\times2}$ & Membership matrix\\
$\mathbb{1}(\cdot)$ & Indicator function, $\mathbb{1}(p) = 1$ if claim $p$ is true, otherwise $\mathbb{1}(p) = 0$\\
$\mathbb{1}_{\ca}$ & Binary indicator vector for community $\ca$, $\mathbb{1}_u = 1$ if $u\in \ca$  otherwise $\mathbb{1}_u = 0$\\
$\mathbb{1}_{\cb}$ & Binary indicator vector for community $\cb$, $\mathbb{1}_u = 1$ if $u\in \cb$  otherwise $\mathbb{1}_u = 0$\\
$g_n = {o}(f_n)$ &  $g_n$ is asymptotically dominated by $f_n$, i.e., $\displaystyle  \lim_{n \rightarrow \infty} \frac{g_n}{f_n} =0$\\
$g_n = {O}(f_n)$ & $g_n$ is asymptotically bounded above by $f_n$, i.e., $\displaystyle  \limsup_{n \rightarrow \infty} \frac{g_n}{f_n} < \infty$\\
$g_n = \Theta(f_n)$ & $\displaystyle  \limsup_{n \rightarrow \infty} \frac{g_n}{f_n} < \infty$ and $\displaystyle  \liminf_{n \rightarrow \infty} \frac{g_n}{f_n} >0$\\
$g_n = \Omega(f_n)$ & $g_n$ bounded below by $f_n$ asymptotically, i.e., $\displaystyle  \limsup_{n \rightarrow \infty} \frac{g_n}{f_n} >0$\\
$g_n = \omega(f_n)$ & $g_n$ dominate $f_n$ asymptotically, i.e., $\displaystyle  \limsup_{n \rightarrow \infty} \frac{g_n}{f_n} = \infty$\\
$\prob(\cdot)$ & Probability measure\\
$\E[\cdot]$ & Expectation \\
\bottomrule
\caption{Summary on notations} 
\label{tab:sumNot}
\end{longtable}

%% file: sections/appx_proof_MLE.tex
{We detail the derivation of the} optimization problem \eqref{opt:MLE} from the maximum likelihood estimator. To start with, we explicitly express the likelihood function as a matrix, which simply relies on subdividing the likelihood function according to which community an edge belongs to.

\begin{restatable}{lemmma}{MLEonDSBM}
\label{lemma:MLE}
Consider a directed graph with adjacency matrix $A$ sampled from the model DSBM$(n_1,n_2,p,q,\eta)$. Then,  applying the maximum likelihood estimation is equivalent to solving the following 
{combinatorial} optimization problem
    \begin{align*}
        \max \quad & \frac{1}{2} \left\langle M_{intra}, \mathbb{1}_{\ca} \mathbb{1}_{\ca}^T + \mathbb{1}_{\cb} \mathbb{1}_{\cb}^T\right\rangle
        + \left\langle M_{inter}, \mathbb{1}_{\ca} \mathbb{1}_{\cb} ^T\right\rangle \tag{MLE}\label{opt:MLE} \\
        s.t. \quad &\mathbb{1}_{\ca} \in \{0,1\}^N\\
        &\mathbb{1}_{\ca} + \mathbb{1}_{\cb} = \mathbb{1}
    \end{align*}
    where $\mathbb{1}_{\ca}, \mathbb{1}_{\cb} \in \{0,1\}^N$ are the indicator vectors for cluster $\ca$ and $\cb$ separately, and
    \begin{align*}
        M_{intra} &= \log{(1/2p)}(A+A^T) + \log{(1-p)} (J - I-A-A^T), \\
        M_{inter} &= \log{(q(1-\eta))} A + \log{(\eta q)} A^T + \log(1-q) (J -I -A-A^T), 
    \end{align*}
    are derived from the log-likelihood functions for intra-community and inter-community edges.
\end{restatable}
\begin{proof}
    Let $A$ be the adjacency matrix of a directed graph generated from DSBM$(n_1,n_2,p,q,\eta)$. For a particular 
    {clusterization of}
    the graph, we use $c$ to denote its community labeling function $c: \mathcal{V} \rightarrow \{\ca, \cb\}$, and we  use the vectors $\mathbb{1}_{\ca}, \mathbb{1}_{\cb} \in \{0,1\}^N$ to indicate community  $\ca$ and $\cb$ separately, where $\mathbb{1}_{\ca}+ \mathbb{1}_{\cb} = \mathbb{1}$. The log likelihood function of $A$ given  $\mathbb{1}_{\ca}$ and $ \mathbb{1}_{\cb}$ can be decomposed as follows 
    \begin{align}          
          \mathcal{L} (A;\sigma) &= \log{\prob(A|\mathbb{1}_{\ca},\mathbb{1}_{\cb})} = \sum_{u<v} \log{\prob(A_{uv}|c(u), c(v))}\nonumber\\
          \label{eq-MLE:1}
          & = \sum_{\substack{u<v\\c(u)=c(v)}} \log{\prob(A_{uv}|c(u), c(v))} + \sum_{\substack{u<v\\c(u)=\ca, c(v)=\cb}} \log{\prob(A_{uv}|c(u), c(v) )},
    \end{align}
    where the first term in \eqref{eq-MLE:1} is only summing over intra-community pair, and the second term 
    {handles} 
    the inter-community pair.

    For an intra-community vertex pair $u,v$, the log-likelihood function is
    \begin{align*}
        \log{\prob(A_{uv}|c(u)=c(v))} = 
        \begin{cases}
            \log{(1/2 p)} \quad &\text{if } u\rightsquigarrow v,\\
            \log{(1/2 p)} \quad &\text{if } v\rightsquigarrow u,\\
            \log{(1-p)} \quad &\text{if } u\not\sim v.
        \end{cases}
    \end{align*}
    This intra-community log-likelihood function coincides with the matrix 
    \begin{align}
    \label{eq:Mintra}
        M_{intra} \triangleq \log{(1/2p)}(A+A^T) + \log{(1-p)} (J-I-A-A^T), 
    \end{align}
    on the entries that represent intra-community pairs, thus allowing us to convert the intra-community summation in \eqref{eq-MLE:1} into the following matrix multiplication form
    \begin{align}
        \label{eq-MLE:2}
        \sum_{\substack{u<v\\c(u)=c(v)}} \log{\prob(A_{uv}|c(u), c(v))}  =  \frac{1}{2} \left\langle M_{intra}, \mathbb{1}_{\ca} \mathbb{1}_{\ca}^T + \mathbb{1}_{\cb} \mathbb{1}_{\cb}^T\right\rangle.
    \end{align}

    For an inter-community vertex $u \in \ca, v\in \cb$, the log-likelihood function is
    \begin{align*}
        \log{\prob(A_{uv}|c(u)=\ca, c(v)=\cb)} = 
        \begin{cases}
            \log{((1-\eta) q)} \quad &\text{if } u\rightsquigarrow v,\\
            \log{(\eta q)} \quad &\text{if } v\rightsquigarrow u,\\
            \log{(1-q)} \quad &\text{if } u\not\sim v.
        \end{cases}
    \end{align*}
    Similar to 
    {the approach  followed} 
    for the intra-community case, we convert the inter-community summation in \eqref{eq-MLE:1} into the  matrix multiplication form
    \begin{align}
        \label{eq-MLE:3}
        \sum_{\substack{u<v\\c(u)=\ca, c(v)=\cb}} \log{\prob(A_{uv}|c(u), c(v))}  = \left\langle M_{inter}, \mathbb{1}_{\ca} \mathbb{1}_{\cb}^T\right\rangle.
    \end{align}
    where
    \begin{align}
    \label{eq:Minter}
        M_{inter} \triangleq \log{((1-\eta)q)}A + \log{(\eta q)} A^T + \log{(1-q)} (J-I-A-A^T).
    \end{align}
   
    Combining \eqref{eq-MLE:2}, \eqref{eq-MLE:3} and \eqref{eq-MLE:1}, we have 
    \begin{align}            
        \log{\prob(A|\mathbb{1}_{\ca},\mathbb{1}_{\cb})}= \frac{1}{2} \left\langle M_{intra}, \mathbb{1}_{\ca} \mathbb{1}_{\ca}^T + \mathbb{1}_{\cb} \mathbb{1}_{\cb}^T\right\rangle
          + \left\langle M_{inter}, \mathbb{1}_{\ca} \mathbb{1}_{\cb} ^T\right\rangle.
    \end{align}
\end{proof}

To 
{arrive at a more compact expression for} the optimization formulation, we introduce an equivalent Hermitian matrix optimization framework. The transformation from the real-valued matrix optimization to the Hermitian optimization builds on the following observation.

\begin{restatable}{lemma}{HermQuad}
\label{lemma:quad_form}
    Consider an arbitrary Hermitian matrix $H = \Re(H) + \iu \Im(H)$, where $\Re(H) \in \R^{n\times n}$ with all $0$ diagonal entries is symmetric, and  $\Im(H) \in \R^{n\times n}$ is skew-symmetric. Let $\bold{x}\in \{\iu ,1\}^n$ be the complex community indicator vector, where $\bold{x}_u = \iu$ for $u \in \ca$. Then, the quadratic form $\bold{x}^*H\bold{x}$ is the sum of entries in $\Re(H)$ that are in the same community, plus the sum of entries in $\Im(H)$ that belong  to different communities, i.e., 
    \begin{align*}
        \bold{x}^*H\bold{x} = \hspace{-12pt}\sum_{\substack{u,v \in \ca\\ \text{or } u,v  \in \cb}} \hspace{-12pt}\Re(H)_{uv} 
        + \hspace{-18pt}\sum_{\substack{u \in \ca v \in \cb\\ \text{or } u\in \cb, v \in \ca}} \hspace{-18pt}\Im(H)_{uv} = 2 \hspace{-12pt}\sum_{\substack{u<v\\u,v \in \ca\\ \text{or } u,v \in \cb}} \hspace{-12pt}\Re(H)_{uv} 
        + 2 \hspace{-18pt}\sum_{\substack{u<v,\\ u \in \ca v \in \cb\\ \text{or } u\in \cb, v \in \ca}} \hspace{-18pt}\Im(H)_{uv}.
    \end{align*}
    In other words,
    \begin{align}
    \label{eq:obs}
        \bold{x^*} H \bold{x} = \langle\Re(H),\mathbb{1}_{\ca} \mathbb{1}_{\ca}^T + \mathbb{1}_{\cb} \mathbb{1}_{\cb}^T \rangle 
        + 2 \langle \Im(H), \mathbb{1}_{\ca} \mathbb{1}_{\cb}^T \rangle. 
    \end{align}
\end{restatable}
With the real-valued MLE optimization formula derived in Lemma~\ref{lemma:MLE} and the observation made in Lemma~\ref{lemma:quad_form}, we are ready to prove the Hermitian optimization formulation of the MLE on DSBM in Theorem~\ref{thm:MLE}.
\ThmHermMLE*
\begin{proof}
    From Lemma~\ref{lemma:MLE}, we derive the log-likelihood 
    \begin{align}            
    \label{eq:ll}
        \log{\prob(A|\mathbb{1}_{\ca},\mathbb{1}_{\cb})}= \frac{1}{2} \left\langle M_{intra}, \mathbb{1}_{\ca} \mathbb{1}_{\ca}^T + \mathbb{1}_{\cb} \mathbb{1}_{\cb}^T\right\rangle
          + \left\langle M_{inter}, \mathbb{1}_{\ca} \mathbb{1}_{\cb} ^T\right\rangle.
    \end{align}
    Compared this equation with the observation \eqref{eq:obs} in Lemma~\ref{lemma:quad_form}, we need the matrix corresponding to the intra-cluster log-likelihood matrix to be a symmetric matrix and need the inter-cluster log-likelihood matrix to be a skew-symmetric matrix, so that we can directly apply \eqref{eq:obs} to convert the real-valued objective into a more compact complex-valued expression. 

    From the definition in \eqref{eq:Mintra} and \eqref{eq:Minter}, we have that $\llin = \llin ^T$, however the inter-cluster log-likelihood matrix $\llout$ is not slew-symmetric. Note that we can be decomposed  the inter-cluster log-likelihood matrix $\llout$ into a symmetric matrix plus a skew-symmetric  matrix as follows
    \begin{align*}
        \llout = \frac{1}{2}(\llout + \llout^T) + \frac{1}{2} (\llout-\llout^T),
    \end{align*}
    where $\llout + \llout^T$ is symmetric and $\llout-\llout^T$ is skew-symmetric.
    Correspondingly, we have
    \begin{align}
    \label{eq:thm_01}
        \left\langle M_{inter}, \mathbb{1}_{\ca} \mathbb{1}_{\cb} ^T\right\rangle = 
        \frac{1}{2}\langle \llout + \llout^T, \mathbb{1}_{\ca} \mathbb{1}_{\cb} ^T \rangle 
        + \frac{1}{2} \langle \llout-\llout^T, \mathbb{1}_{\ca} \mathbb{1}_{\cb} ^T\rangle,
    \end{align}
    where the first term sums over the inter-cluster entries of a symmetric matrix and the second term sums over the inter-cluster entries of a skew-symmetric matrix.  Due to the symmetry in the first term of \eqref{eq:thm_01}, we can further write it as
    \begin{align}
    \label{eq:thm_02}
        \frac{1}{2}\langle \llout + \llout^T, \mathbb{1}_{\ca} \mathbb{1}_{\cb} ^T \rangle  = \frac{1}{4}\langle \llout + \llout^T, J - \mathbb{1}_{\ca} \mathbb{1}_{\ca} ^T -  \mathbb{1}_{\cb} \mathbb{1}_{\cb} ^T \rangle 
    \end{align}
    Combining \eqref{eq:thm_01}, \eqref{eq:thm_02} and \eqref{eq:ll}, we can rewrite the log-likelihood objective as
    \begin{align*}
    &\log{\prob(A|\mathbb{1}_{\ca},\mathbb{1}_{\cb})} = \frac{1}{2} \left\langle M_{intra}, \mathbb{1}_{\ca} \mathbb{1}_{\ca}^T + \mathbb{1}_{\cb} \mathbb{1}_{\cb}^T\right\rangle
          + \left\langle M_{inter}, \mathbb{1}_{\ca} \mathbb{1}_{\cb} ^T\right\rangle\\
        &= \frac{1}{2} \left\langle M_{intra}, \mathbb{1}_{\ca} \mathbb{1}_{\ca}^T + \mathbb{1}_{\cb} \mathbb{1}_{\cb}^T\right\rangle + 
         \frac{1}{2}\langle \llout + \llout^T, \mathbb{1}_{\ca} \mathbb{1}_{\cb} ^T \rangle 
        + \frac{1}{2} \langle \llout-\llout^T, \mathbb{1}_{\ca} \mathbb{1}_{\cb} ^T\rangle\\
        &=\frac{1}{2} \left\langle M_{intra}, \mathbb{1}_{\ca} \mathbb{1}_{\ca}^T + \mathbb{1}_{\cb} \mathbb{1}_{\cb}^T\right\rangle +\frac{1}{4}\langle \llout + \llout^T, J - \mathbb{1}_{\ca} \mathbb{1}_{\ca} ^T -  \mathbb{1}_{\cb} \mathbb{1}_{\cb} ^T \rangle  \\
        & \quad + \frac{1}{2} \langle \llout-\llout^T, \mathbb{1}_{\ca} \mathbb{1}_{\cb} ^T\rangle
    \end{align*}
    Because the term $\langle \llout+\llout^T, J \rangle$ is always a constant and resealing the objective function by a constant factor $4$ does not affect the optimal solution, therefore solving the \eqref{opt:MLE} in Lemma~\ref{lemma:MLE} is equivalent to solve the following
    \begin{align*}
        \max \quad &  \left\langle 2M_{intra}- (\llout+\llout^T), \mathbb{1}_{\ca} \mathbb{1}_{\ca}^T + \mathbb{1}_{\cb} \mathbb{1}_{\cb}^T\right\rangle
        + 2\left\langle M_{inter} - \llout^T, \mathbb{1}_{\ca} \mathbb{1}_{\cb} ^T\right\rangle  \\
        s.t. \quad &\mathbb{1}_{\ca} \in \{0,1\}^N\\
        &\mathbb{1}_{\ca} + \mathbb{1}_{\cb} = \mathbb{1}
    \end{align*}
Using \eqref{eq:obs} from Lemma~\ref{lemma:quad_form}, we convert the above real-valued optimization problem to the following complex-valued equivalence
\begin{align*}
        \max \quad &  \bold{x}^*\Hmle \bold{x}  \\
        s.t. \quad &\bold{x} \in \{\iu,1\}^N\\
    \end{align*}
where the Hermitian matrix $\Hmle$ has
\begin{align*}
    \Re(\Hmle) &=  2M_{intra}- (\llout+\llout^T)\\
    &=  \log{\frac{p^2}{4\eta(1-\eta)q^2}}(A+A^T) + 2\log{\frac{1-p}{1-q}} (J -I -A- A^T)\\
    & = \log\left( \frac{p^2(1-q)^2}{4\eta(1-\eta)q^2(1-p)^2}\right) (A+A^T) + 2\log\left( \frac{1-p}{1-q}\right)(J-I),\\
    \Im(\Hmle) & = \ M_{inter} - \llout^T\\
    & = \log\left(\frac{1-\eta}{\eta} \right) (A-A^T).
\end{align*}
\end{proof}

%% file: sections/appx_proof_perturb.tex
\subsection{Proof overview}
The intuition behind the success of spectral relaxation clustering is that the expected Hermitian matrix $\E[\Hmle]$ has community-dependent structure, and its leading eigenvector $\bold{v}^* = \bold{v}_1(\E[\Hmle])$ exactly encodes the true community labels through two distinct values. In practice, we observe the empirical matrix $\Hmle$, which serves as a perturbed version of $\E[\Hmle]$. Classical matrix perturbation theory ensures that if $\Hmle$ is sufficiently close to $\E[\Hmle]$, the leading eigenvector $\topv = \bold{v}_1(\Hmle)$ remains informative and enables accurate recovery of the community structure. Our proof is structured as follows:
\begin{enumerate}[leftmargin = 18pt]
    \item[(i)] 
    Using matrix perturbation theory, we first bound the eigenvector perturbation $ \|{\topvE}{\topvE}^* - \topv \topv^*\|_F$;
    \item[(ii)] 
    Using the Matrix-Bernstein inequality from random matrix theory, we provide a high-probability upper bound on the random perturbation $\| \Hmle - \E[\Hmle]\|$;  
    \item[(iii)] 
    Combining the results from the above two steps, we perform an error analysis on the $k$-means clustering step, and derive the final spectral clustering error bound.
\end{enumerate}

We first characterize the eigenspace properties of $\E[\Hmle]$ in the following Lemma~\ref{lemma:EH_eig}.

\begin{restatable}{lemma}{LemmaEHeig}
    \label{lemma:EH_eig}
For the DSBM$(n_1,n_2,p,q,\eta)$, the population matrix $\E[\Hmle]$ has a unique largest eigenvalue. The top eigenvector $\topvE$ has exactly two distinct values that indicate the community labels, where the distance $d$ between them can be easily computed using \eqref{eq:centroids}.
Moreover, the eigengap is
\begin{align}
    |\lambda_1(\E[\Hmle]) -\lambda_2(\E[\Hmle])| =  \min \{2\Delta, |1/2 N(w_rp+w_c)| + \Delta\} \geq  \Delta, 
\end{align}
where 
\begin{align}
    \label{eq:eigengap}
    \Delta  = 1/2\sqrt{N^2(w_rp+w_c)^2 -4n_1n_2\left((w_rp+w_c)^2 - |w_r+w_c+\iu w_i(1-2\eta)q|^2\right)}.
\end{align}
\end{restatable}

We consider $\Hmle$ as a perturbed version of $\E[\Hmle]$, and denote $R = \Hmle - \E[\Hmle]$. The perturbation on the eigenspace and eigenvalues is characterized by the following two well-known {results, namely the} Davis-Kahan perturbation bound (Theorem~\ref{lemma:D-K}) and Weyl's inequality (Theorem~\ref{thm:wyle}), from which we derive an upper bound on the eigenspace {misalignment} distance $ \|{\topvE}{\topvE}^* - \topv \topv^*\|_F$.
\begin{restatable}{lemma}{LemTopEigPert}
    \label{lemma:top-eig_pert}
Given a directed graph from DSBM$(n_1,n_2, p,q,\eta)$ and its Hermitian matrix representation $\Hmle$, the projection matrix of the top eigenvector is such that
    \begin{align*}
        \|{\topvE}{\topvE}^* - \topv \topv^*\|_F \leq 2\sqrt{2}\frac{\|R\|}{\lambda_1(\E[\Hmle]) - \lambda_2(\E[\Hmle])}.
    \end{align*}
\end{restatable}

We apply the Matrix Bernstein inequality {for} Hermitian matrices, and obtain a upper bound on $\|R\|$. 
\begin{restatable}{lemma}{LemBoundPert}
\label{lemma:bound_perturb}
Consider a directed graph from DSBM $(n_1,n_2,p,q,\eta)$ and its Hermitian matrix representation $\Hmle$. Assume that $N\pmax = \Omega(\log{N})$.
Then there exists an absolute constant  $\epsilon$ and
\begin{align}
    \label{eq:C}
          C = (2+\epsilon)\sqrt{w_r^2 + w_i^2}\left(\frac{\log N}{N\pmax} + 1\right)= \Theta\left(\sqrt{w_r^2 + w_i^2}\right),
    \end{align}
such that the random perturbation $R = \Hmle - \E[\Hmle]$ has 
    \begin{align*}
        \prob(\|R\| \geq C\sqrt{N \pmax \log{N}}) \leq N^{-\epsilon}.
    \end{align*}
\end{restatable}

\subsection{Error analysis Theorem~\ref{thm:errorHermSC}}
\label{appx:proof_thm2}
\ThmErrorSC*
\begin{proof}
    Recall that the key steps of our spectral clustering algorithm involve: first compute the top eigenvector of $\Hmle$, and then cluster the vertices using $k$-means on the embedding space given by the real and imaginary part of the top eigenvector. We use $\hat{U} = [\Re(\topv) , \Im(\topv)] $ to denote the embedding space given by the concatenation of the real and imaginary part of the top eigenvector of $\Hmle$ and ${U} = [\Re(\topvE) , \Im(\topvE)] $ for that of $\E[\Hmle]$, where both $U,\hat{U}\in \R ^{N\times2}$. For the clustering outcomes, we denote by $\hat{\sigma}_{\text{SC-MLE}}$ the clustering result using $\Hmle$, and use $\sigma$ to represent the clustering given by $\E[\Hmle]$. First, note that from Lemma~\ref{lemma:EH_eig}, we conclude that the leading eigenvector of $\E[\Hmle]$ perfectly recovers the true community membership, therefore $\sigma$ is the true community membership vector.
    Next, given that the $k$-means clustering step achieves a $(1+\epsilon)$ approximation, using the error bound \eqref{eq:err-kmeans} from Lemma~\ref{lemma:km_error}, we have that 
    \begin{align*}
        l(\sigma, \hat{\sigma}_{\text{SC-MLE}}) d^2 \leq 4(4+2\epsilon) \|\hat{U} - U\|_F^2, 
    \end{align*}
    where $d$ is the distance between the two cluster centroids of the population version $\E[\Hmle]$, with its expression provided in \eqref{eq:centroids}. 
    {Given that a rotation of} $\hat{U}$ does not change the $k$-means clustering result,  the tightest upper bound we can obtain is 
    \begin{align}
    \label{eq:thm2_1}
        l(\sigma, \hat{\sigma}_{\text{SC-MLE}}) d^2 &\leq 4(4+2\epsilon) \min_{O\in \mathcal{O}_2}\|\hat{U} - OU\|_F^2 \\
        \label{eq:thm2_3}
        &= 4(4+2\epsilon)\min_{r\in \C_1} \|\topvE-r\topv\|_F^2 \leq  4(4+2\epsilon)\| \topv \topv^* - {\topvE} {\topvE^*}   \|^2_{F},
    \end{align}
    where  \eqref{eq:thm2_1} follows from the fact that $\|U - \hat{U}\|_F = \|\topvE - \topv\|_F$ and the inequality in \eqref{eq:thm2_3} follows from Lemma~\ref{lem:dist_comp}.

    Combining matrix perturbation analysis from Lemma~\ref{lemma:EH_eig}, Lemma~\ref{lemma:top-eig_pert} and Lemma~\ref{lemma:bound_perturb}, we {derive that, for an} absolute constant $\epsilon_0$,  with probability at least $1- N^{-\epsilon_0}$,  {the misalignment distance is upper bounded by}
    \begin{align}
    \label{eq:thm2_2}
        \|\topv \topv^* - {\topvE} {\topvE^*}\|_F \leq \frac{2\sqrt{2}C\sqrt{N\pmax \log{N}} }{ {\lambda_1(\E[\Hmle]) - \lambda_2(\E[\Hmle])}  } \leq \frac{2\sqrt{2}C\sqrt{N\pmax \log{N}} }{ \Delta  } .
    \end{align}
    
    Combining \eqref{eq:thm2_3} and \eqref{eq:thm2_2}, we eventually obtain that with probability at least $1- N^{-\epsilon_0}$
    \begin{align*}
         \frac{l(\sigma, \hat{\sigma}_{\text{SC-MLE}})}{N}\leq  \frac{64(2+\epsilon)C^2{\pmax \log{N}} }{ d^2 \Delta^2 }.
    \end{align*}
\end{proof}

\subsection{Eigenspace of the population matrix $\E[\Hmle]$)}
\label{appx:eig_EH}
\LemmaEHeig*
\begin{proof}
Recall that the population version of $\Hmle$ has a block structure and can be written as 
\begin{align*}
    \E[\Hmle] &= MQM^T -(pw_r+w_c)I\\
    Q &= \begin{bmatrix}
0 & 1 \\
-1 & 0 
\end{bmatrix}
+ w_r \begin{bmatrix}
p & q \\
q & p
\end{bmatrix}
 + w_c 
 \begin{bmatrix}
1 & 1 \\
1 & 1 
\end{bmatrix},
\end{align*}
where $M \in \{0,1\}^{N\times2}$ is the community 
{membership  matrix,  and $M_{uc} =1$ denotes that} vertex $u$ belongs to community $c$. We further normalize the columns of $M$ as follows 
\begin{align*}
    MQM^T&= MD^{-1} DQD (MD^{-1})^T \\
    D &= \begin{bmatrix}
        \sqrt{n_1} & 0\\
        0 & \sqrt{n_2}.
    \end{bmatrix}
\end{align*}
Here the normalized matrix $MD^{-1}$ has orthonormal column vectors.

Let $ DQD = U\Lambda U^*$ be the eigendecomposition on the
$2\times2$ matrix, 
Then, the $N\times N$ matrix $MQM^T$ can be diagonalized as  
\begin{align*}
   MQM^T =   (MD^{-1} U) \Lambda (MD^{-1} U)^*,
\end{align*}
where $\diag(\Lambda)$ contains the eigenvalues of  $MQM^T$ and the columns of $MD^{-1} U \in \R^{N\times2}$ are the orthonormal eigenvectors.
Therefore, the problem of computing the eigenpairs of $\E[\Hmle]$ reduces to compute the eigenpairs of the $2\times2$ matrix $DQD$ where
    \begin{align*}
        DQD = \begin{bmatrix}
            n_1 (w_rp+w_c) & \sqrt{n_1n_2} (w_rq +w_c +\iu (1-2\eta)q)\\
            \sqrt{n_1n_2} (w_rq +w_c -\iu (1-2\eta)q) & n_2(w_rp+w_c)
        \end{bmatrix}
    \end{align*}

For the eigenvalues, via a simple calculation we arrive at  
\begin{align*}
    \lambda_1(DQD) &= \frac{1}{2} \left( N(w_rp+w_c) + 2\Delta \right),\\
    \lambda_2(DQD)& = \frac{1}{2} \left( N(w_rp+w_c) - 2\Delta \right),
\end{align*}
where
\begin{align*}
    2\Delta &= \sqrt{N^2(w_rp+w_c)^2 -4n_1n_2((w_rp+w_c)^2 - |w_rq+w_c+\iu w_iq(1-2\eta)|^2)}.
\end{align*}
Therefore, we obtain the eigenvalues of $\E[H] = MQM^T$
\begin{align*}
    \lambda_1(\E[\Hmle]) &= \frac{1}{2} \left( N(w_rp+w_c) + 2\Delta \right) -(w_rp+w_c) \\
    \lambda_2(\E[\Hmle]) &= \frac{1}{2} \left( N(w_rp+w_c) - 2\Delta \right) -(w_rp+w_c)\\
    \lambda_3(\E[\Hmle])  = \ldots = \lambda_N(\E[\Hmle]) &=-(w_rp+w_c).
\end{align*}
The eigenvalue that obtains the largest magnitude is unique and it is $\lambda_1(\E[\Hmle])$ when $N(w_rp+w_c)\geq 0$, or $\lambda_2(\E[\Hmle])$ when $N(w_rp+w_c) < 0$.
The gap between the largest and the second largest eigenvalue is 
\begin{align}
      \min \{2\Delta, |1/2 N(w_rp+w_c)| + \Delta\}.
\end{align}
One can easily verify that the eigengap lies in $[ {\Delta}, 2\Delta]$.
Therefore, the lower bound ${\Delta}$ is a good approximation to the spectral gap in the sense that they are of the same order.

Next, we move on to compute the top eigenvector of $\E[\Hmle]$.
We use  $\bold{x} = (x_1,x_2) \in \C^2$ to denote the top eigenvector of $DQD$, and we have that $x_1 \neq x_2$.
Then, the top eigenvector of $\E[\Hmle]$ can be easily computed through $\topvE = MD^{-1}\bold{x}$, and it has two distinct values
\begin{align}
\label{eq:centroids_val}
    \topvE(u) = 
    \begin{cases}
        x_1/ \sqrt{n_1} \;\ \text{if } u \in \ca,\\
        x_2 / \sqrt{n_2}\;\ \text{if } u \in \cb.
    \end{cases}
\end{align}
The distance between the two cluster centroids $d$ is simply
\begin{align}
    \label{eq:centroids}
    d = \left|\frac{x_1}{\sqrt{n_1}}  - \frac{x_2}{\sqrt{n_2}} \right|.
\end{align}
\end{proof}
\subsection{Useful theorems from matrix perturbation analysis}

\begin{theorem}[Davis-Kahan's perturbation bound \cite{davis-1970-perturbation}]
\label{lemma:D-K}
Let  $H,  R \in \Herm$ be two Hermitian matrices. Then, for any $a \leq \beta$ and $\delta>0$ it holds that
\begin{align*}
    \left\|P_{[\alpha, \beta]}(H)-P_{(\alpha-\delta, \beta+\delta)}(H+R)\right\| \leq \frac{\|R\|}{\delta}.
\end{align*}
Here $P_{[\alpha, \beta]}(H)$ denotes the projection matrix on the subspace spanned by eigenvectors of $H$ with corresponding eigenvalues lie between $[\alpha, \beta]$, and $P_{(\alpha-\delta, \beta+\delta)}(H+R)$ is the projection matrix on the subspace spanned by eigenvectors of $H+R$ with eigenvalues lie between $(\alpha - \delta, \beta + \delta)$. 
\end{theorem}

\begin{theorem}[Weyl's inequality \cite{weyl-1912-asymptotische}]
\label{thm:wyle}
Let ${H},{R} \in \Herm$ be two Hermitian matrices. Then for every $1 \leq j \leq n$, the $j$-th largest eigenvalues of $H$ and $H+R$ obey
$$
\left|\lambda_j(H)-\lambda_j(H+R)\right| \leq\|R\| .
$$
\end{theorem}

In addition to the eigenspace perturbation bound in Themrem~\ref{lemma:D-K}, we summarize in Lemma~\ref{lem:dist_comp} comparisons over two different representations of the eigenspace distance, which will useful in the error analysis of $k$-means.

\begin{lemma}
\label{lem:dist_comp}[Adapted From Lemma~2.1 in \cite{chen-2021-spectral}]
For any $U, \tilde{U} \in \C^{\N\times k}$, we have 
    \begin{align*}
        \min_{O\in \mathcal{O}_{k\times k}} \| U - O\tilde{U}\|_F \leq \|UU^* - \tilde{U}\tilde{U}^*\|_F
    \end{align*}
\end{lemma}

\subsection{Eigenspace perturbation}
\label{appx:pert_topv}
\LemTopEigPert*
\begin{proof}[Proof on Lemma~\ref{lemma:top-eig_pert}]
From the Davis-Kahan's pertubation bound, we have
\begin{align}
    \label{eq:lem9_1}
    \| {\topv} {\topv}^* - \topvE \topvE^* \|  \leq \frac{\|R\|} {|\lambda_1(\E[\Hmle]) -\lambda_2(\Hmle)|}.
\end{align}
Using Wyle's inequality, we have that
\begin{align*}
    |\lambda_2(\E[\Hmle])-\lambda_2(\Hmle)|  \leq \|R\|
\end{align*}
Therefore, the denominator in \eqref{eq:lem9_1} can be further lower bounded by $|\lambda_1(\E[\Hmle]) - \lambda_2(\E[\Hmle])| - \|R\|$, and we obtain
\begin{align}
    \label{eq:lem09_2}
    \| {\topv} {\topv}^* - \topvE \topvE^* \|  \leq \frac{\|R\|} {|\lambda_1(\E[\Hmle]) -\lambda_2(\E[\Hmle])| - \|R\|}.
\end{align}
The denominator in \eqref{eq:lem09_2}  involves comparing the spectral gap $|\lambda_1(\E[\Hmle]) -\lambda_2(\E[\Hmle])|$ and $\|R\|$, which further requires an extra condition on the denominator being positive, to allow the inequality to hold.  To circumvent this limitation, we divide the comparison into two cases 
\begin{itemize}
    \item if $\|R\| \geq \frac{1}{2} |\lambda_1(\E[\Hmle]) -\lambda_2(\E[\Hmle])|$, then we have
    \begin{align*}
         \| {\topv} {\topv}^* - \topvE \topvE^* \|  \leq 1 \leq  \frac{2\|R\|} {|\lambda_1(\E[\Hmle]) -\lambda_2(\E[\Hmle])|}.
    \end{align*}
    \item  if $\|R\| \leq \frac{1}{2} |\lambda_1(\E[\Hmle]) -\lambda_2(\E[\Hmle])|$, then we use the perturbation bound \eqref{eq:lem09_2} and get
    \begin{align*}
         \| {\topv} {\topv}^* - \topvE \topvE^* \|  \leq \frac{\|R\|} {|\lambda_1(\E[\Hmle]) -\lambda_2(\E[\Hmle])| - \|R\|} \leq  \frac{2\|R\|} {|\lambda_1(\E[\Hmle]) -\lambda_2(\E[\Hmle])|}.
    \end{align*}
\end{itemize}
Combining the two cases, we obtain that for any $\|R\|$, the following upper bound always holds 
\begin{align*}
    \| {\topv} {\topv}^* - \topvE \topvE^* \|  \leq   \frac{2\|R\|} {|\lambda_1(\E[\Hmle]) -\lambda_2(\E[\Hmle])|}.
\end{align*}
Since for any rank $r$ Hermitian matrix $H$, $\|H\|_F \leq  \sqrt{r} \|H\|$.
    Therefore, we have that 
    \begin{align*}
        \|\topvE \topvE^* - {\topv} {\topv}^*\|_F \leq  \sqrt{2} \|\topvE \topvE^* - {\topv} {\topv}^*\|  \leq \frac{2\sqrt{2}\|R\|} {|\lambda_1(\E[\Hmle]) -\lambda_2(\E[\Hmle])|}.
    \end{align*}
\end{proof}
\subsection{Bound the random perturbation $\|R\|$ }
\label{appx:boundR}
\begin{lemma}[Matrix Bernstein \cite{tropp-2015-concentration_inq}]
\label{lemma:bernstein}
Consider a finite sequence $\left\{{S}_k\right\}$ of independent, random matrices with dimension $d$. Assume that
\begin{align*}
    \E {S}_k={0} \text { and }\left\|{S}_k\right\| \leq L \text { for each index } k .
\end{align*}
For the random matrix ${Z}=\sum_k {S}_k$, let $v({Z})$ be the matrix variance statistic of the sum:
\begin{align*}
v({Z})  =\max \left\{\left\|\sum_k \mathbb{E}\left({S}_k {S}_k^*\right)\right\|,\left\|\sum_k \mathbb{E}\left({S}_k^* {S}_k\right)\right\|\right\} .
\end{align*}
Then, for all $t \geq 0$,
\begin{align*}
    \mathbb{P}\{\|{Z}\| \geq t\} \leq 2d \exp \left(\frac{-t^2 / 2}{v({Z})+L t / 3}\right) .
\end{align*}
\end{lemma}
\LemBoundPert*
\begin{proof}
    Recall that by definition random perturbation $R = \Hmle - \E[\Hmle]$ is Hermitian. 
    We first decompose it into summation of perturbations on different entries $R = \sum_{j<l} R^{jl}$ where $R^{jl}$ is also a random Hermitian and only has non-zero entries at $(j,l)$ and $(l,j)$. 
    If $j,l$ belongs to the same community $\sigma(j) = \sigma(l)$
        \begin{align}
        \label{eq:lem5_1}
            R^{jl}_{jl}= 
            \begin{cases}
                w_r(1-p)+\iu w_i &\;\ w.p. \;\ p/2 \\
                w_r(1-p)-\iu w_i & \;\ w.p. \;\ p/2 \\
                -w_r p &\;\ w.p. \;\ 1-p. 
            \end{cases}
        \end{align}
        If $j,l$ belongs to different communities $\sigma(j) \neq \sigma(l)$, and without loss of generality we assume $j\in \ca, l \in \cb$
        \begin{align}
        \label{eq:lem5_2}
            R^{jl}_{jl}= 
            \begin{cases}
                w_r(1-q)+\iu w_i (1-(1-2\eta)q) &\;\ w.p. \;\ q(1-\eta) \\
               w_r(1-q)-\iu w_i (1+(1-2\eta)q) & \;\ w.p. \;\ q\eta \\
                -w_r q - \iu w_i (1-2\eta)q &\;\ w.p. \;\ 1-q
            \end{cases}
        \end{align}
        From the Matrix Bernstein's inequality in Lemma~\ref{lemma:bernstein},  we have for any $t\geq 0$
        \begin{align*}
            \prob(\|R\| \geq t) \leq 2N \exp{\left( \frac{-t^2/2}{\var(R) + Lt/3}\right)},
        \end{align*}
        where $L$ is an upper upper of $\|R^{jl}\|$ and $\var(R)$ is the variance.  

    For computing $L$, recall that by definition 
    \begin{align*}
        \|R^{jl}\| \leq L, \;\ \forall j\neq l.
    \end{align*}
    Here the matrix spectral norm can be simplified to be upper bounded by 
    $ |R_{jl}|$ because the spectral norm is always upper bounded by the maximum absolute values of each entry. Therefore, it suffices to take $L$ as an upper bound on $\max_{j \neq l}{|R_{jl}|}$. From \eqref{eq:lem5_1}, we have that, if $\sigma(j) = \sigma(l)$, then $|R_{jl}| \leq \sqrt{w_r^2(1-p)^2 + w_i^2}$; if $\sigma(j) \neq \sigma(l)$, then $|R_{jl}| \leq \sqrt{ w_r^2(1-q)^2 + w_i^2(1+(1-2\eta)q)^2}.$ Combining the two, it suffices for us  to take 
    \begin{align}
    \label{eq:L}
        L = \sqrt{w_r^2(1-p_{min})^2 + w_i^2(1+(1-2\eta)q)^2 }
        \leq 2 \sqrt{w_r^2+w_i^2}.
    \end{align}

    To compute the variance term $\var(R)$, first recall by definition
        \begin{align*}
        \var(R) = \max \left\{ \left\| \sum_{j<l} \E [R^{jl} (R^{jl})^* ] \right\| , \left\| \sum_{j<l} \E [ (R^{jl})^* R^{jl} ]  \right\| \right\}.
        \end{align*}
    For each $j<l$, we have $R^{jl} (R^{jl})^* = (R^{jl})^* R^{jl}$ and we use $M^{jl}$ to denote the product matrix.  In $M^{jl}$, the only two non zero entries are $M^{jl}_{jj}$ and $M^{jl}_{ll}$ and $M^{jl}_{jj} =M^{jl}_{ll} = R_{jl} \overline{R_{jl}}$. Therefore, $\E [R^{jl} (R^{jl})^* ]$ also only has two non-zero entries at $(j,j)$ and $(l,l)$ for every $j<l$ and thus, the spectral norm of the matrix summation is simply the largest diagonal element, i.e.,
    \begin{align}
    \label{eq:lem5_3}
        \var(R) = \max_{ j \in [N]} \sum_{l\neq j} \E[M^{jl}_{jj}].
    \end{align}
    In \eqref{eq:lem5_3}, $M^{jl}_{jj}$ is a real random variable whose distribution can be derived from \eqref{eq:lem5_1} and \eqref{eq:lem5_2}, and we have that, if $\sigma(j) = \sigma(l)$, then
    \begin{align*}
        M_{jj}^{jl} = 
        \begin{cases}
            w_r^2(1-p)^2 + w_i^2  &\;\ w.p. \;\ p \\
            w_r^2p^2 &\;\ w.p. \;\ 1-p, 
        \end{cases}
    \end{align*}
    and $\E[M_{jj}^{jl}] = w_r^2p(1-p) + pw_i^2.$
    
    If $\sigma(j) \neq \sigma(l)$, then 
    \begin{align*}
        M_{jj}^{jl} = 
        \begin{cases}
            w_r^2(1-q)^2 + w_i^2(1-(1-2\eta)q)^2 &\;\ w.p. \;\ q(1-\eta) \\
            w_r^2(1-q)^2 + w_i^2(1+(1-2\eta)q)^2 &\;\ w.p. \;\ q\eta \\
            w_r^2q^2 + w_i^2(1-2\eta)^2q^2 &\;\ w.p. 1-q, 
        \end{cases}
    \end{align*}
    and $\E[M_{jj}^{jl}] = w_r^2q(1-q) + w_i^2q(1-(1-2\eta)^2q)$. Since, without loss of generality, we assume $p,q\leq 0.5$, thus for all $j\neq l$ we have
    $\E[M_{jj}^{jl}] \leq w_r^2\pmax(1-\pmax) + w_i^2 \pmax$. Therefore, from \eqref{eq:lem5_3}, we arrive at   
    \begin{align}
    \label{eq:var}
        \var(R) \leq N  (w_r^2\pmax(1-\pmax) + w_i^2 \pmax) \leq N\pmax(w_r^2+w_i^2).
    \end{align}

Using the Matrix Bernstein's inequality, we have for $t= C\sqrt{N\pmax\log N}$,
\begin{align}
    \nonumber
    \prob(\|R\| \geq t) &\leq 2\exp{\left( - \frac{C^2 N\pmax\log{N}}{2\var(R) + 2LC\sqrt{N\pmax\log{N}}/3} + \log{N} \right)}\\
    \label{eq:lem5_4}
    &\leq 2\exp{\left( - \frac{C^2N\pmax\log{N}}{ 2N\pmax (w_r^2+w_i^2)+ 2L \sqrt{N\pmax\log{N}}/3} + \log{N} \right)}\\
    \nonumber
    & = 2\exp{\left( - \frac{C^2}{ 2(w_r^2+w_i^2)+ 2LC/3 \sqrt{\log{N}/N\pmax}} \log{N}+ \log{N} \right)}.
\end{align}
Here \eqref{eq:lem5_4} follows from the analysis on $\var(R)$ in  \eqref{eq:var}. From \eqref{eq:lem5_5}, if there exist an absolute $\epsilon$, such that
\begin{align}
        \label{eq:lem5_5}
     \frac{C^2}{ (w_r^2+w_i^2)+ LC/3 \sqrt{\log{N}/N\pmax}} \geq 2+\epsilon, 
\end{align}
then we have $\prob(\|R\| \geq t) \leq N^{-\epsilon}$, which conclude the proof.  
It turns out that we can always find an absolute constant $C$ such that  \eqref{eq:lem5_5} holds. 
To see this, first note that \eqref{eq:lem5_5} is equivalent to 
\begin{align*}
    C \geq (1+\epsilon/2) L\sqrt{\frac{\log N}{N\pmax}} + \sqrt{(2+\epsilon) (w_r^2+w_i^2) + (1+\epsilon/2)^2 L^2 \frac{\log N}{N\pmax}}.   
\end{align*}
Since $a^2+b^2\leq (a+b)^2$ for $a,b>0$, it suffices to let  
\begin{align*}
    C = (2+\epsilon) L\frac{\log N}{N\pmax} + (2+\epsilon)\sqrt{w_r^2 + w_i^2}.
\end{align*}
Since $L \leq 2 \sqrt{w_r^2+w_i^2}$, we have
\begin{align}
    C \leq (2+\epsilon) \sqrt{w_r^2 + w_i^2}\left(\frac{\log N}{N\pmax} + 1\right)= \Theta\left(\sqrt{w_r^2 + w_i^2}\right),
\end{align}
where the last equality is due to the connectivity assumption $N\pmax = \Omega(\log{N})$.


\end{proof}

\subsection{Useful theorem in $k$-means error analysis}
\begin{lemma}[$k$-means error adapted from Lemma~5.3 in \cite{lei-2015-consistency}]
\label{lemma:km_error}
    For $\epsilon>0$ and any two matrices $\hat{U},U$, such that $U = MX$ with $M \in \{0,1\}^{N\times2}$ be the indicator matrix and $X \in \R^{2\times 2}$ have its row vectors representing the centroids of two clusters, let $(\hat{M}, \hat{X})$ be a $(1+\epsilon)$ solution to the $k$-means problem and $\Bar{U} = \hat{M} \hat{X}$. For $\delta = \|X_{1*}-X_{2*}\|$, define $S = \{ j \in [N]:\|\Bar{U}_{j*} - U_{j*}\|\} \geq \delta/2$ then
    \begin{align}
        \label{eq:err-kmeans}
        |S| \delta^2 \leq 4(4+2\epsilon) \|\hat{U} - U\|^2_{F}.
    \end{align}
\end{lemma}



\subsection{Proof on Corollary~\ref{coro:eta}}
\Coroeta*
\begin{proof}
    From Theorem~\ref{thm:errorHermSC}, we have the general upper bound on the error bound 
    \begin{align}
    \label{eq_coro2_1}
        \frac{l(\sigma, \hat{\sigma})}{N}\leq  \frac{64(2+\epsilon)C^2{\pmax \log{N}} }{ d^2\Delta^2 }= \Theta\left(\frac{C^2 \pmax \log N}{d^2 \Delta^2}\right),
    \end{align}
    where $d$ and $\Delta$ depends on $\E[\Hmle]$. When $p=q$, we have that 
    \begin{align*}
        w_r = \log \left( \frac{1}{4\eta(1-\eta)}\right), \quad w_i = \log{ \left(\frac{1-\eta}{\eta}\right)}, \quad w_c =0.
    \end{align*}
    Moreover, notice that normalizing $\Hmle$ {does not} affect the clustering error. For the rest of the discussion, we consider $1/w_i \Hmle$ as the input Hermitian matrix for spectral relaxation of MLE, and correspondingly we denote the updated coefficient as
    \begin{align*}
        \Tilde{w_r} = \log \left( \frac{1}{4\eta(1-\eta)}\right) /\log{\left( \frac{1-\eta}{\eta}\right)}, \quad \Tilde{w_i} = 1, \quad \Tilde{w_c} =0.
    \end{align*}
    Because $\Tilde{w_r}\leq 1 $ and $\Tilde{w_i}=1$, the term $C^2$ in \eqref{eq_coro2_1} has $C^2 = \Theta(\Tilde{w_r}^2+\Tilde{w_i}^2) = \Theta(1)$. 
    Therefore, we can further simplify \eqref{eq_coro2_1}  as follows 
    \begin{align*}
        \frac{l(\sigma,\hat{\sigma})}{N} \leq \Theta\left(\frac{p\log{N}}{d^2\Delta^2}\right).
    \end{align*}
    For analyzing the asymptotic behaviour of the error bound, we are only left with computing the centroid distance $d$ and eigengap bound $\Delta$.
    
    Following from the definition of $\Delta$ in \eqref{eq:eigengap}, we have 
    \begin{align}
    \label{eq_coro2_2}
        \Delta = \frac{Np}{2} \sqrt{\Tilde{w_r}^2 + \Tilde{w_i}^2(1-2\eta)^2}.
    \end{align}
    For computing the centroid distance $d$, recall that the population matrix can be written as  
    \begin{align*}
        \E[\Hmle] = MQM^T - \Tilde{w_r} p I,
    \end{align*}
    where $M$ is the community indicator matrix and  the $2\times2$ matrix  $Q$ has 
    \begin{align*}
        Q = 
        \begin{bmatrix}
            \Tilde{w_r} &\Tilde{w_r} + (1-2\eta)\iu\\
             \Tilde{w_r} - (1-2\eta)\iu & \Tilde{w_r}
        \end{bmatrix}p.
    \end{align*}
    Since $n_1 = n_2 = N/2$, we have that the two distinct values in  $v_1(\E[\Hmle])$ (see \eqref{eq:centroids_val}) to be the values of $v_1(Q)$ divided by $\sqrt{N/2}$.  We can easily compute that the top eigenvector  has
    \begin{align}
    \label{eq:coro1_1}
        v_1 (\E[\Hmle])= \begin{cases}
             \frac{w_r + w_i(1-2\eta) \iu }{\sqrt{N}|-w_r + w_i(1-2\eta) \iu|} \quad &\text{for } u \in \ca\\
            1/\sqrt{N} \quad &\text{for } u \in \cb.
        \end{cases}
    \end{align}
    We denote $\bar{c}_1, \bar{c_2}$ the cluster centroids of $\ca, \cb$ in the embedding given by the top eigenvector of $\E[\Hmle]$. The locations of two cluster centroids in the complex plane are exactly the two distinct values in \eqref{eq:coro1_1}. We visualize the two cluster centroids in Figure~\ref{fig:LTheta}. Let $\theta = \arccos \left(\frac{w_r}{|w_r + w_i(1-2\eta) \iu|} \right)$ be the angle between the two values in the complex plane. Therefore, we have 
    \begin{align}
    \label{eq_coro2_3}
        d^2 = \frac{1}{N}(1-\cos\theta) = \frac{4\sin^2\theta/2}{N}.
    \end{align}
    Combining \eqref{eq_coro2_1}, \eqref{eq_coro2_2}, and \eqref{eq_coro2_3} and letting 
    \begin{align}
    \label{eq:L_eta_1}
        L(\eta) =|\Tilde{w_r}+\iu (1-2\eta)| \sin (\theta/2).
    \end{align}
     We have 
    \begin{align}
    \label{eq:coro2_4}
        \frac{l(\sigma, \hat{\sigma})}{N}\leq \Theta\left( \frac{\log{N}}{NpL^2}\right).
    \end{align}
    From the above inequality \eqref{eq:coro2_4}, the upper bound of misclustering error is determined by two independent variables $Np$ and $L^2$. The term $Np$ is the average degree of the graph. The term $L^2$, by definition \eqref{eq:L_eta_1}, is a function on $\eta$ with expression 
    \begin{align}
    \label{eq:L_eta}
    L(\eta) = \left( (1 - 2\eta)^2 + \frac{
    \left( \log\left( \frac{1}{4\eta - 4\eta^2} \right) \right)^2}{\left( \log\left( \frac{\eta}{1 - \eta} \right) \right)^2
    } 
\right)
\left(
1 - 
\frac{
\log\left( \frac{1}{4\eta - 4\eta^2} \right)
}{
\log\left( \frac{\eta}{1 - \eta} \right)
\sqrt{
(1 - 2\eta)^2 + 
\frac{
\left( \log\left( \frac{1}{4\eta - 4\eta^2} \right) \right)^2
}{
\left( \log\left( \frac{\eta}{1 - \eta} \right) \right)^2
}
}
}
\right)
\end{align}
    To see how the value $L$ changes as $\eta$ varies from $0$ to $0.5$, we plot  $L(\eta)$ in Figure~\ref{fig:LEta} using Mathematica \cite{Mathematica}. 
    We observe  that $L(\eta) = \Theta(1)$ when $\eta$ is bounded away from $0.5$. Therefore,
    if $\eta\leq 0.5-\epsilon$ with an absolute constant $\epsilon >0$, then the misclustering error on the spectral relaxation is such that 
    \begin{align*}
        \frac{l(\sigma,\hat{\sigma})}{N}  = O\left(\frac{\log{N}}{Np} \right)
    \end{align*}
    When  $\eta$ converges to $0.5$ (when the imbalance structure disappears), $L(\eta)$ converges to $0$. Therefore,
    if $\eta = 0.5 -o(1)$,  we have that 
    \begin{align*}
        \frac{l(\sigma,\hat{\sigma})}{N}  = \omega \left(\frac{\log{N}}{Np} \right).
    \end{align*} 
     The above results on misclustering error bounds agree with the intuition that {lower values of $\eta$ denote a less noisy problem instance, and thus lead to a lower} clustering error. 
    \begin{figure}
    \centering
    \begin{subfigure}[b]{0.49\textwidth}
         \centering
         \includegraphics[width=0.8 \textwidth]{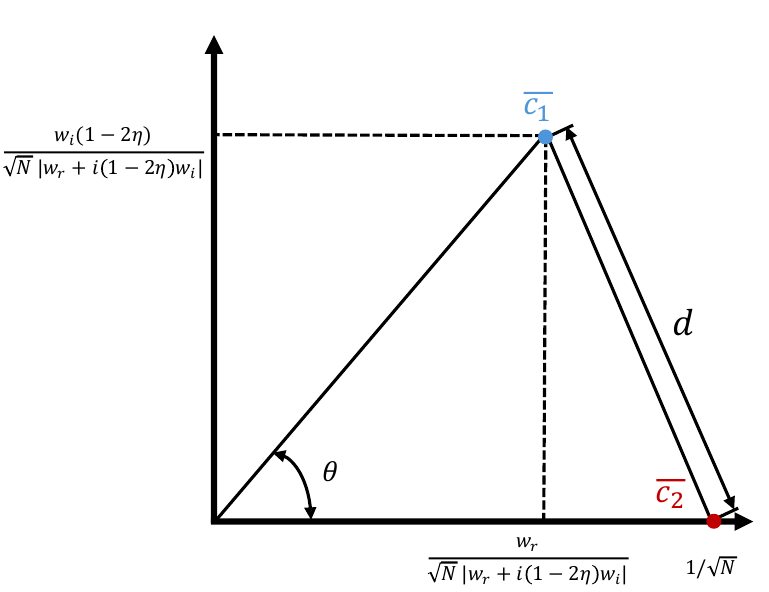}
         \caption{Visualization on $l$ and $\theta$.  
         }
         \label{fig:LTheta}
     \end{subfigure}
     \hfill
     \begin{subfigure}[b]{0.49\textwidth}
         \centering
         \includegraphics[width=0.8\textwidth]{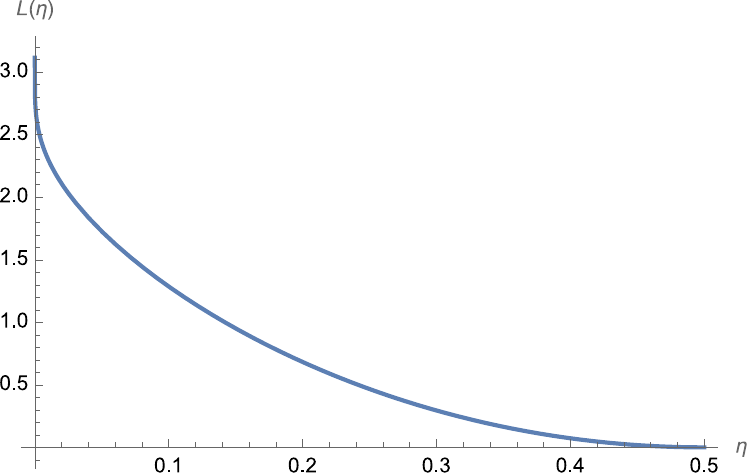}
         \caption{Plot of $l$ as function of $\eta$. 
         }
         \label{fig:LEta}
     \end{subfigure}
     \caption{Visualization of important parameters for representing the error bound.}
\end{figure}

\end{proof}

%% file: sections/appx_convergence.tex
\subsection{Convergence of iterative learning on model parameters}
\label{appx:convergence}

We conduct empirical tests on this iterative approach on directed graphs generated from the DSBM ensemble. In Figure~\ref{fig:iter}, we show how the learned model parameters vary as one repeats the updating process in \texttt{LE-SC}.  Through our study on the synthetic data sets from the DSBM, we observe that in most cases, this iterative algorithm converges near the truth model parameters very fast (within 10 iterations). 
\begin{figure}[thp!]
    \centering
    \includegraphics[width=0.9\textwidth]{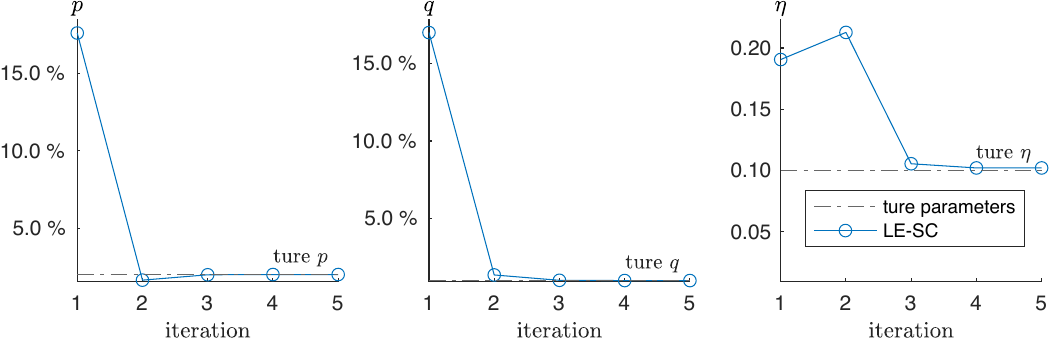}
    \caption{Illustration on the convergence of the iterative algorithm. We first sample a directed graph from DSBM with $n_1=n_2 =1000, p = 2\%, q = 1\%, \eta = 0.1$. Then, starting from a random initialization on the model parameters, we apply \texttt{LE-SC} to learn the DSBM parameters. The lines with circles represent model parameter learning using the spectral clustering algorithm \texttt{LE-SC}.}
    \label{fig:iter}
\end{figure}

We also conduct experiments to examine how different initialization strategies influence the clustering outcomes. Below, we summarize the recommended strategies.
\begin{enumerate}
    \item[a.] When edge direction is the primary consideration for clustering, we recommend initializing the Hermitian matrix as $\Hmle_0 = \iu(A - A^T) + (A + A^T)$, which approximates the MLE in the low direction noise regime ($\eta \approx 0$). Alternatively, one may use $\Hmle_0 = \iu(A - A^T)$, which, as discussed in Section~\ref{sec:optimization}, corresponds to the net flow optimization approach.
    \item[b.] When edge density is the main focus, we suggest initializing with $\Hmle_0 = A + A^T$, as it aligns with the total flow optimization scheme.
    \item [c.]  We also recommend using alternative clustering algorithms to generate initial clusters, such as \texttt{DI-SIM} \cite{rohe-2016-coclustering}, \texttt{BibSym} \citep{satuluri-2011-symmetrizations}, and \texttt{D-SCORE} \cite{wang-2020-dscore}.
\end{enumerate}

%% file: sections/appx_complexity.tex
\subsection{Complexity analysis for eigenvector computation.}
\label{appx:complexity}

The power method is a fast and scalable algorithm for computing the eigenvectors of sparse matrices through iterative updates. 
Our algorithm \texttt{LE-SC} involves computing the eigenvector of a dense Hermitian matrix $\Hmle$ \eqref{eq:H-mle}, which can be decomposed into a sparse part $w_r(A+A^T) + \iu w_i (A-A^T)$ and an all-one component $w_c J$.  To apply the power method to compute the top eigenvector of this Hermitian matrix as follows
\begin{enumerate}
    \item Randomly initialize $b_0$
    \item For $k\geq 1$, $b_k = (w_r(A+A^T) + \iu w_i (A-A^T))b_{k-1} + w_c J b_{k-1}$
    \item Repeat until $\|b_k - b_{k-1}\|_2 \leq \epsilon$
\end{enumerate}
The convergence of the above iterative steps depended linearly on the $\lambda_2(\Hmle)/\lambda_1(\Hmle)$. Each iteration involves a matrix-vector multiplication with time complexity $\mathcal{O}(|\mathcal{E}|)$ for the sparse part and $O({N})$ for the all-one matrix.  Therefore, the overall computational complexity of computing the top eigenvector of $\Hmle$ is $O(|\mathcal{E}|)$, enabling the method to efficiently scale to large and sparse graphs.

%% file: sections/appx_DSBM.tex
\subsection{\texttt{LE-SC} for multiple clusters}
\label{appx:multi-cluster}
\begin{algorithm}[tph!]
\SetKwInOut{Input}{Input}
\SetKwInOut{Output}{Output}
\caption{Multi-Cluster Likelihood Estimation Spectral Clustering \texttt{LE-SC}}
\label{alg:LE-SC-k}
\Input{Directed graph $G = (\mathcal{V}, \mathcal{E})$; target number of clusters $k$; maximum iterations per bipartition $T$}
\Output{Community labels $\hat{\sigma} : \mathcal{V} \rightarrow \{1, \dots, k\}$}
\BlankLine
Initialize clustering $\mathcal{C} \gets \{\mathcal{V}\}$ and label counter $\ell \gets 1$\;
\While{$|\mathcal{C}| < k$}{
    Select the largest cluster $\mathcal{S} \in \mathcal{C}$\;
    Apply \texttt{LE-SC} (Algorithm~\ref{alg:LE-SC}) to subgraph $G[\mathcal{S}]$ with $T$ iterations to bipartition $\mathcal{S}$ into $(\mathcal{S}_1, \mathcal{S}_2)$\;
    Replace $\mathcal{S}$ in $\mathcal{C}$ with $\mathcal{S}_1$ and $\mathcal{S}_2$\;
}
Assign unique labels $1, \dots, k$ to the final clusters in $\mathcal{C}$ to obtain $\hat{\sigma}$\;
\Return{$\hat{\sigma}$}
\end{algorithm}

\subsection{Experiment: DSBM with different higher-order structure}
\label{appx:k-DSBM}
When there are multiple communities in the DSBM, the edge probability between communities can vary and form different higher-order structures. We use a directed meta graph to describe such a higher-order structure of community-communitty interaction.  In a meta graph, each vertex represents a community, and edges between communities are directed and weighted. Each directed edge in the meta graph indicates a structured edge generating probability, where the edge weight is the probability that an edge is oriented from a source community to a target community. A non-edge means the edge direction between the two communities is totally random, i.e, 1/2 probability for each direction.  To give a concrete example, the meta graph in Figure~\ref{fig:5_cycle} is a 5-cycle. The directed edge between $C_1$ and $C_2$ indicates that edges between them are generated with probability $1-\eta$ pointing from $C_1$ to $C_2$ and with probability $\eta$ backwards.

\begin{figure}[]
    \centering
    \begin{subfigure}{0.45\textwidth}
        \includegraphics[width=0.3\textwidth]{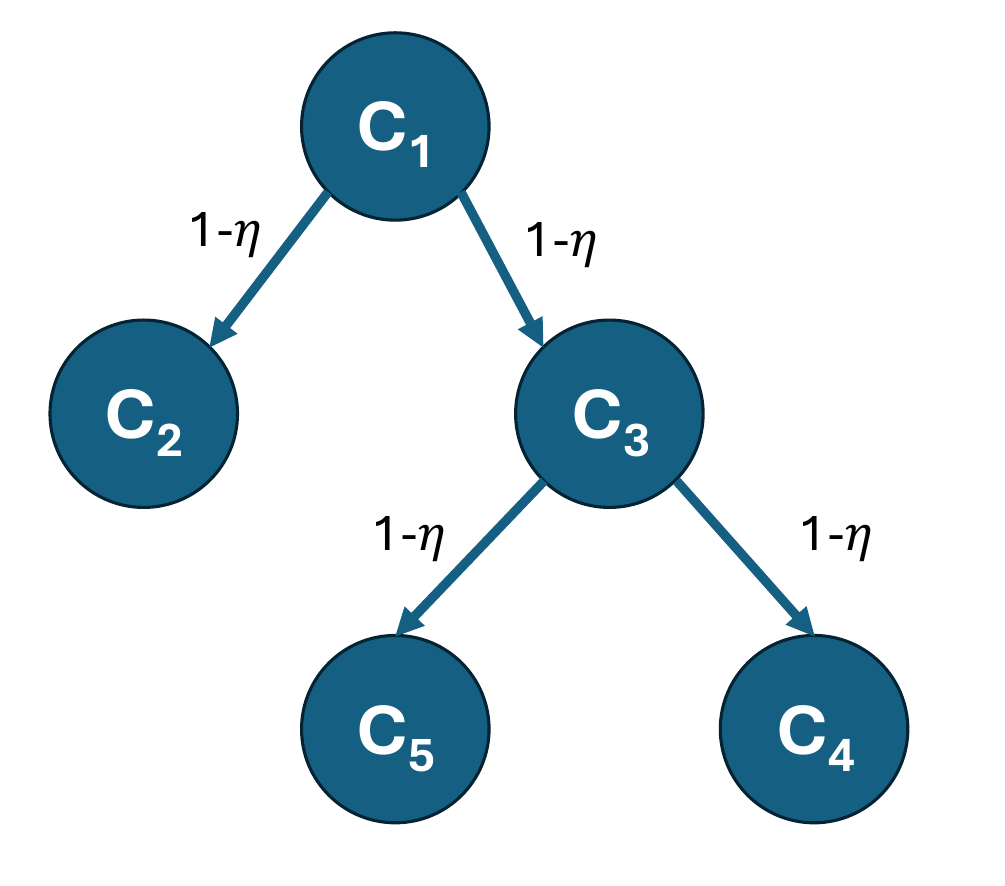}\includegraphics[width=0.5\textwidth]{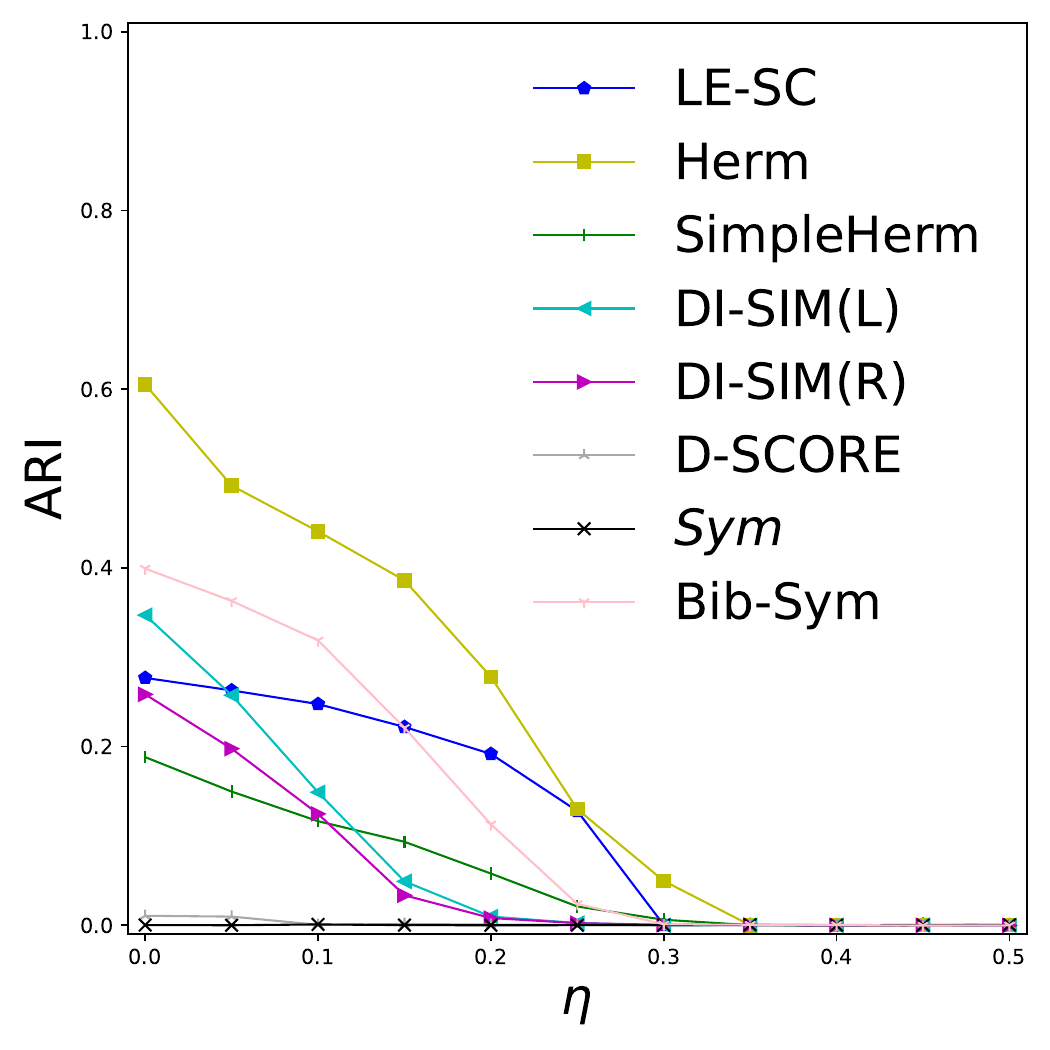}
        \caption{$p = q = 1\%$}
    \end{subfigure}
    \begin{subfigure}{0.45\textwidth}
        \includegraphics[width = 0.3\textwidth]{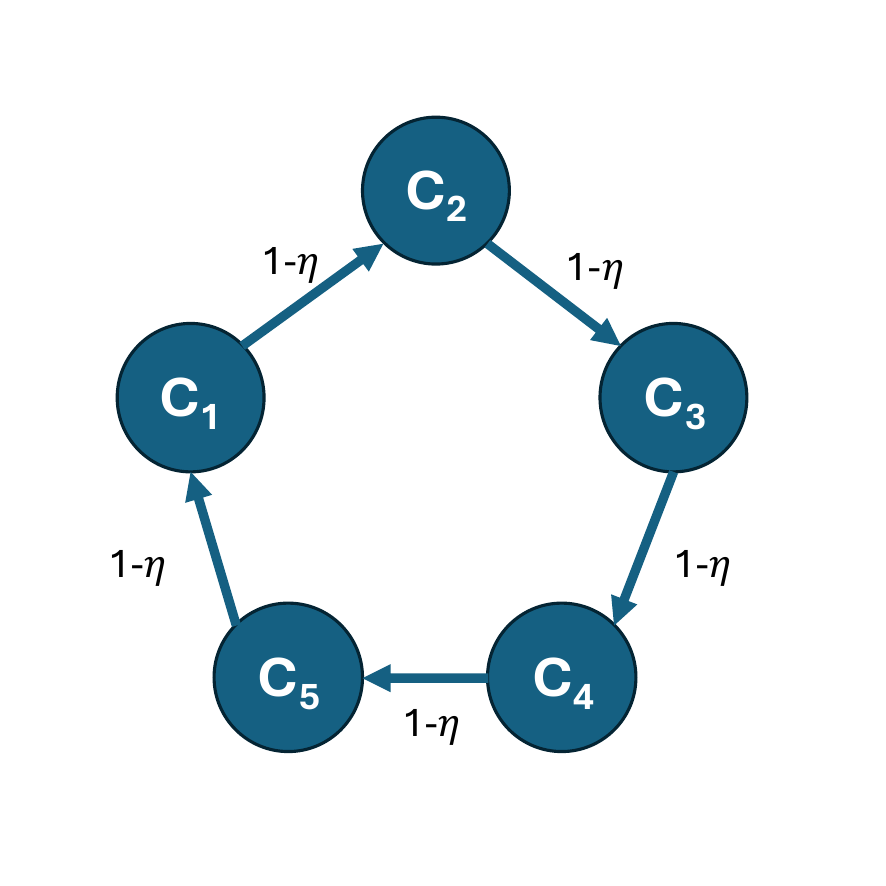}\includegraphics[width=0.5\linewidth]{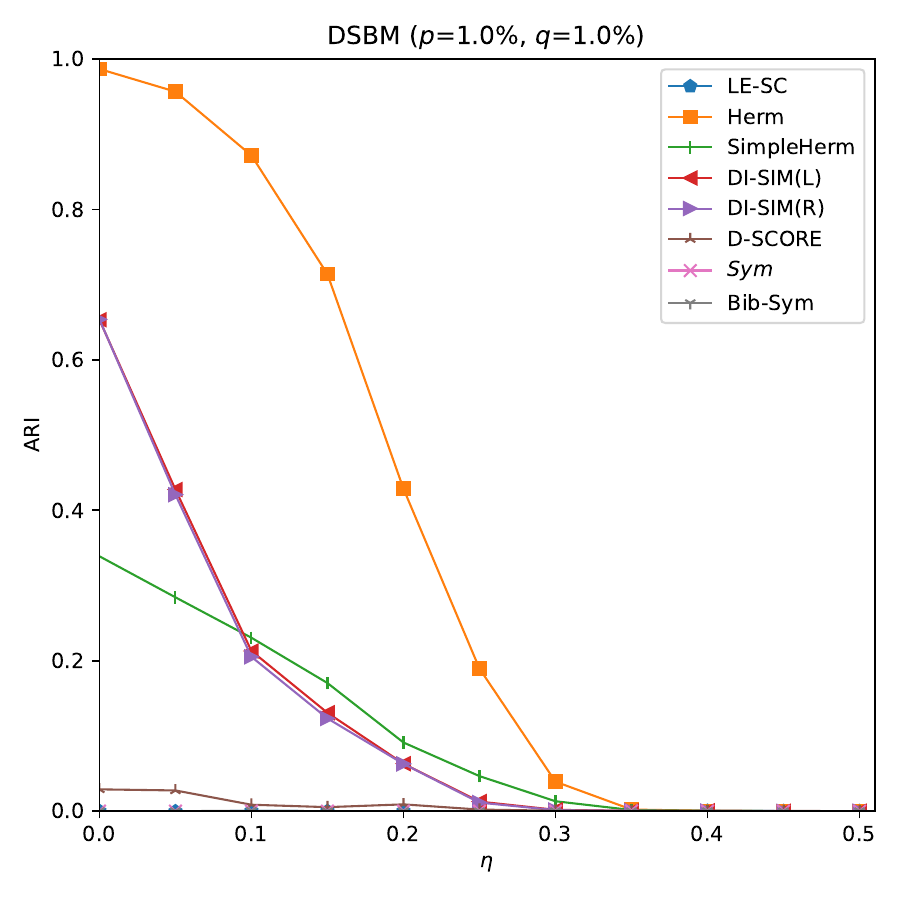}
        \caption{$p = q = 1\%$}
    \end{subfigure}
    \begin{subfigure}{0.45\textwidth}
        \includegraphics[width = 0.4\textwidth]{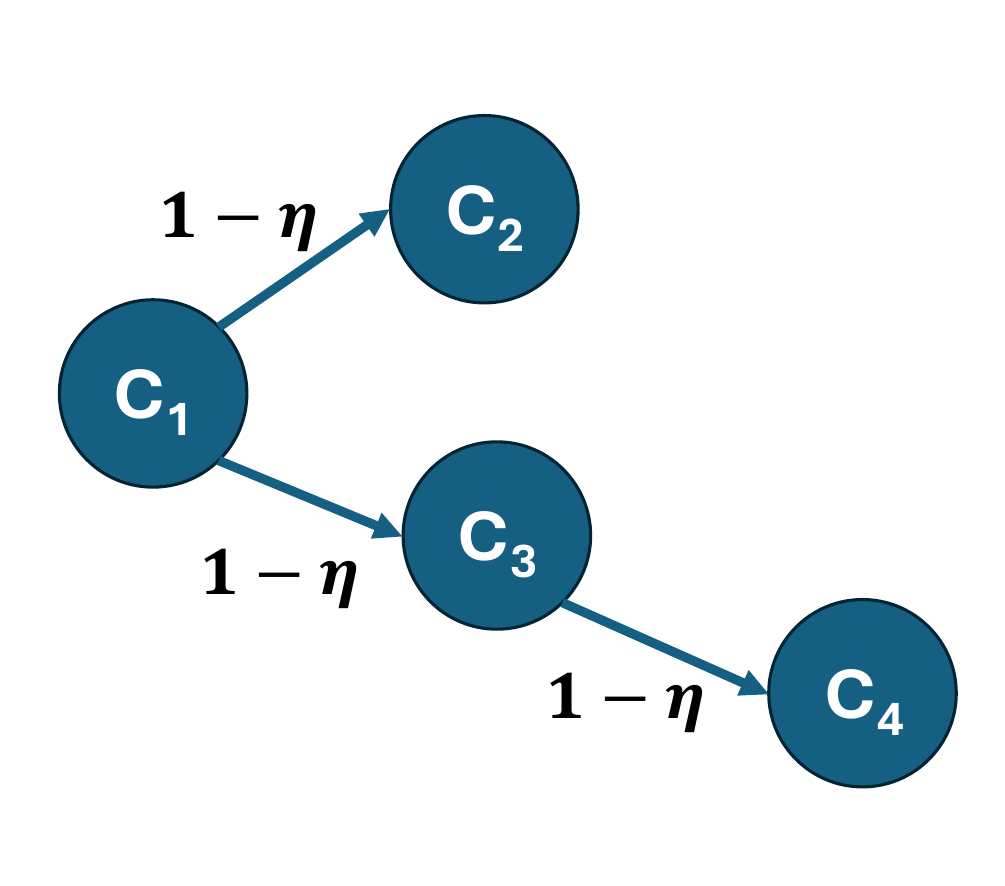}\includegraphics[width=0.5\linewidth]{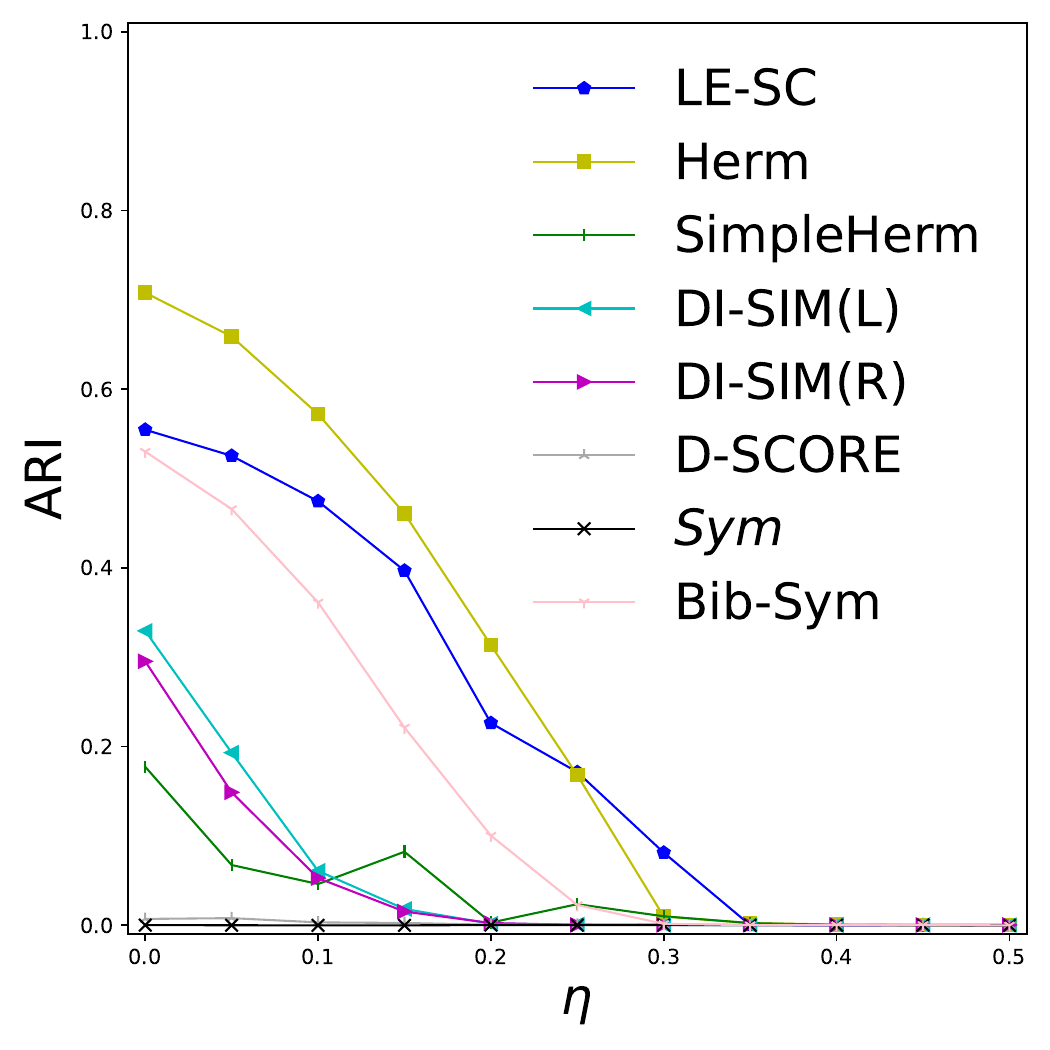}
        \caption{$p = q = 1\%$}
    \end{subfigure}
    \caption{Recovery rate on DSBM ($n_i = 1000$), where communities form different higher order structure.}
    \label{fig:5_cycle}
\end{figure}

We conduct experiments on DSBMs with varying higher-order community structures, represented by different meta-graphs, and under different levels of directional noise $\eta$. In each experiment, we independently sample 5 directed graphs with a fixed parameter set, and report the averaged ARIs over these graph samples.
We observe that our algorithm overall demonstrates competing performance under various multicommunity settings, except in the case of a cycle-structured meta-graph with $p=q$. This limitation arises because our method relies on a bipartition, and when the meta-edges form a cycle, there is no meaningful way to enforce a consistent bipartition.  Addressing this limitation presents an interesting direction for future work.

\subsection{Experiment: US migration data}
\label{sec:appx_US}
We conduct experiments on migration data comprising 3074 counties in the mainland United States. Migration flows are represented as a directed, weighted graph, where each edge weight represents the number of individuals migrating from one county to another.
To avoid a biased result dominated by extremely high degree vertices, we normalize the directed graph and use $D^{-1/2}A D^{-1/2}$, with the degree matrix $D$ accounting for both incoming and outgoing edges. We report the clustering outcomes from different spectral methods.
In our experiments, we replace the baseline method \texttt{Herm} with its variant \texttt{Herm(RW)}, which yields significantly better performance on this dataset (as \texttt{Herm} fails to produce meaningful results). We also observe that \texttt{SimpHerm} yields trivial clusters when $k\geq 4$.

\begin{figure}[tph!]
    \centering
    \begin{subfigure}{0.24\textwidth}
        \includegraphics[width = \textwidth]{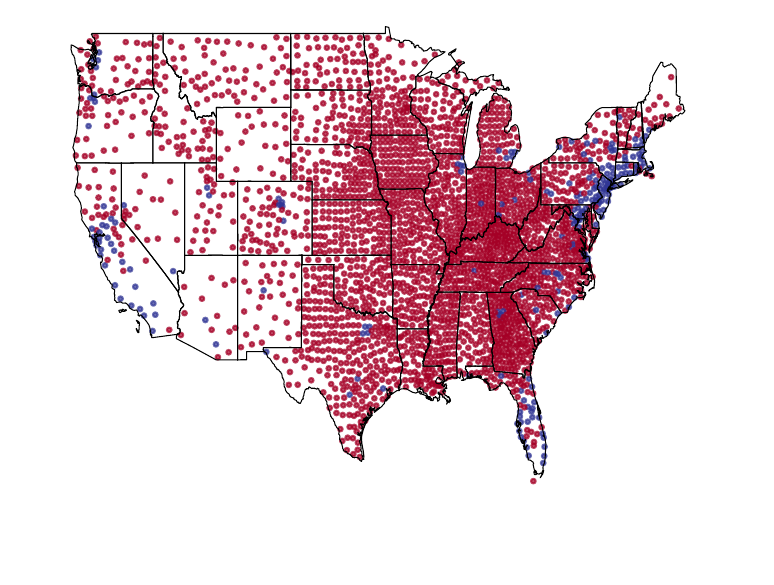}
        \vspace{-24pt}
        \caption{\texttt{LE-SC}}\label{fig:US_LESC}
    \end{subfigure}
    \begin{subfigure}{0.24\textwidth}
        \includegraphics[width = \textwidth]{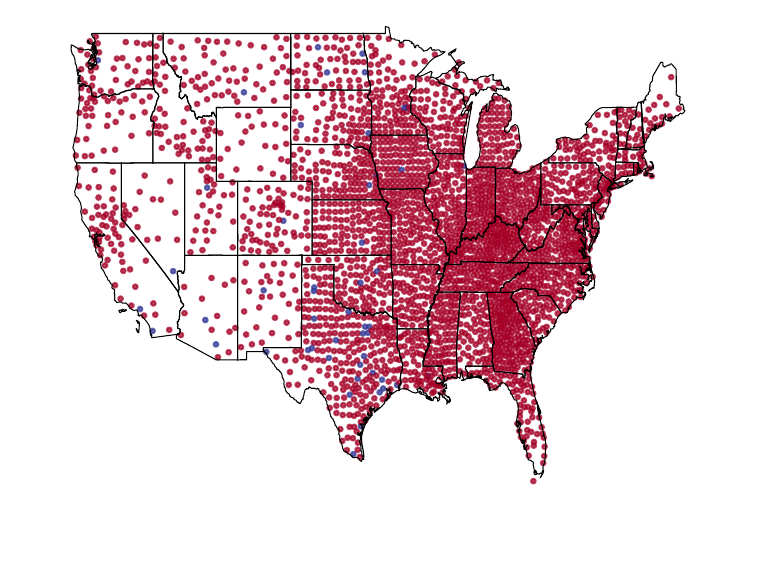}
        \vspace{-24pt}
        \caption{\texttt{Herm(RW)}}
    \end{subfigure}
    \begin{subfigure}{0.24\textwidth}
        \includegraphics[width = \textwidth]{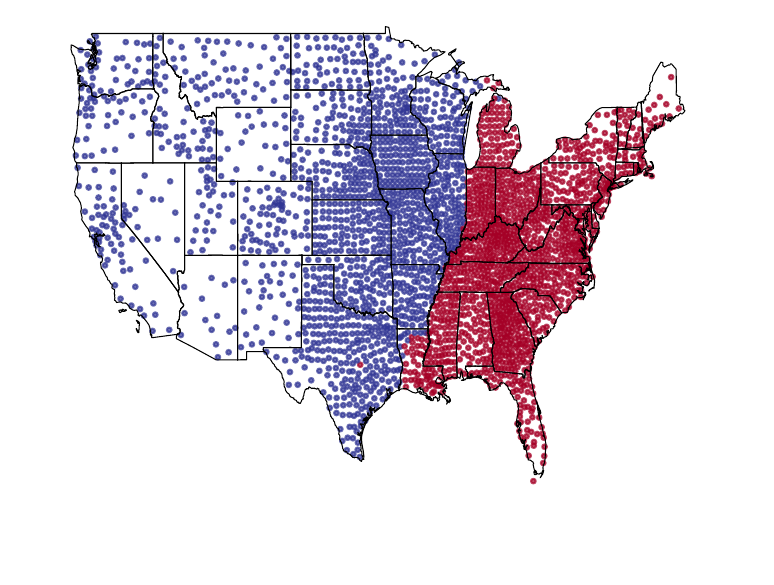}
        \vspace{-24pt}
        \caption{\texttt{DI-SIM(L)}}
    \end{subfigure}
    \begin{subfigure}{0.24\textwidth}
        \includegraphics[width = \textwidth]{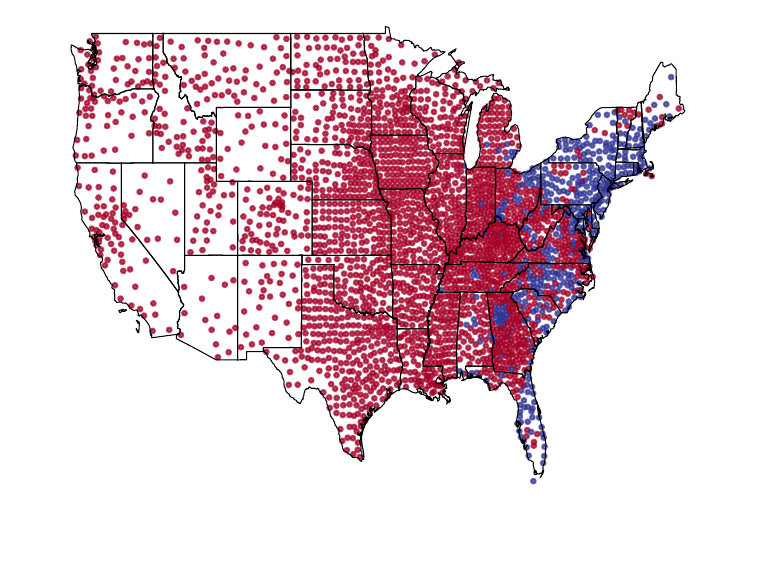}
        \vspace{-24pt}
        \caption{\texttt{Bib-Sym}}
    \end{subfigure}
    \begin{subfigure}{0.24\textwidth}
        \includegraphics[width = \textwidth]{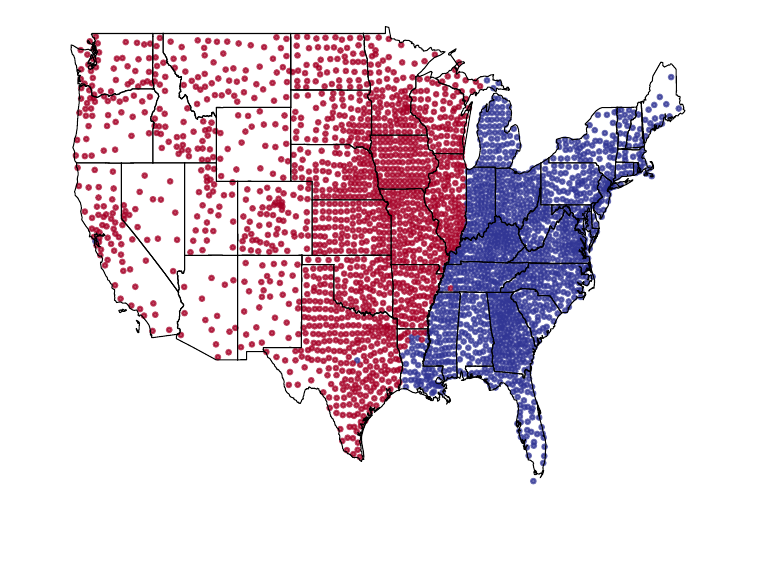}
        \vspace{-24pt}
        \caption{\texttt{DI-SIM(R)}}
    \end{subfigure}
    \begin{subfigure}{0.24\textwidth}
        \includegraphics[width = \textwidth]{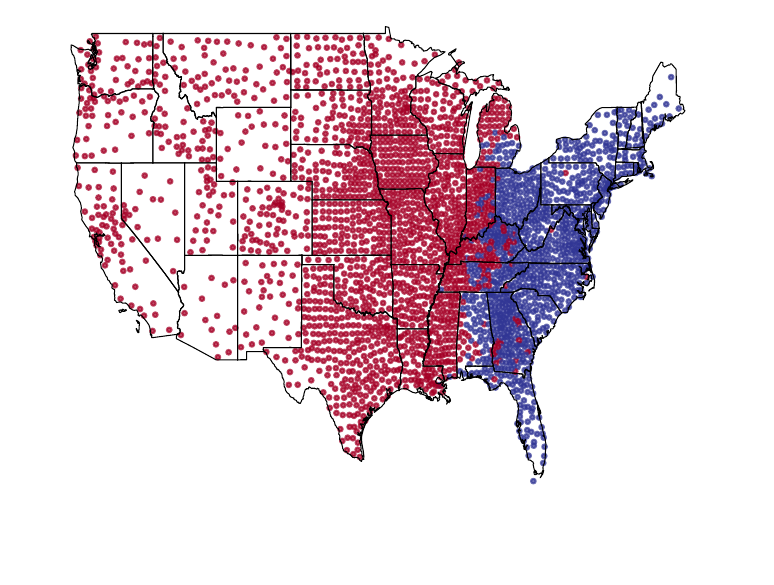}
        \vspace{-24pt}
        \caption{\texttt{Sym}}
    \end{subfigure}
    \begin{subfigure}{0.24\textwidth}
        \includegraphics[width = \textwidth]{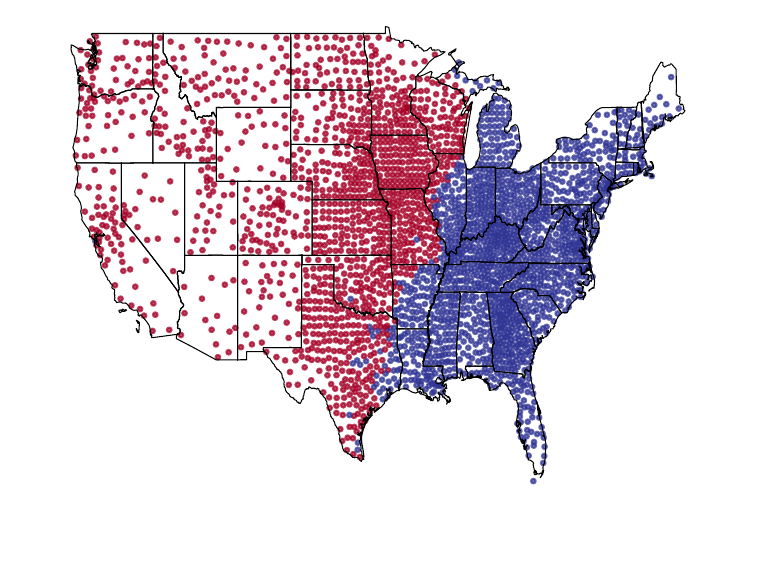}
        \vspace{-24pt}
        \caption{\texttt{SimpHerm}}
    \end{subfigure}
    \begin{subfigure}{0.24\textwidth}
        \includegraphics[width = \textwidth]{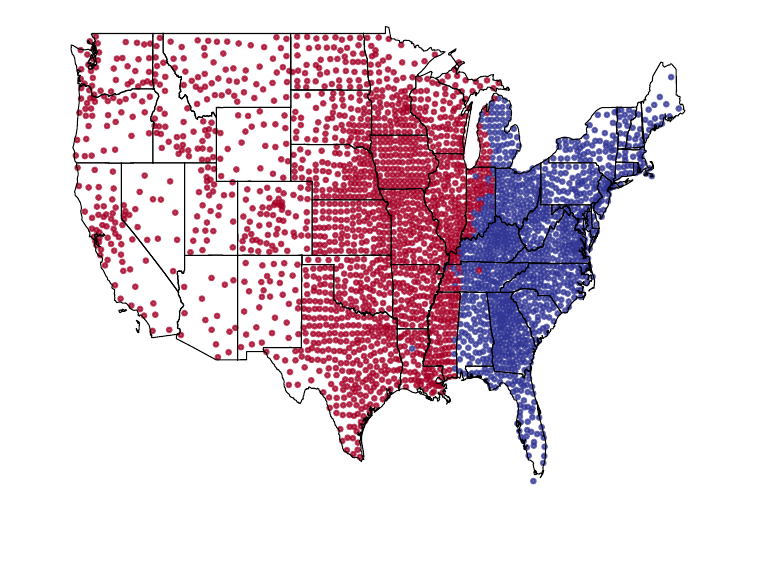}
        \vspace{-24pt}
        \caption{\texttt{D-Score}}
    \end{subfigure}
    \caption{Counties clustered by migration data ($k=2$).}
    \label{fig:migration2}
\end{figure}

\begin{figure}[tph!]
    \centering
    \begin{subfigure}{0.24\textwidth}
        \includegraphics[width = \textwidth]{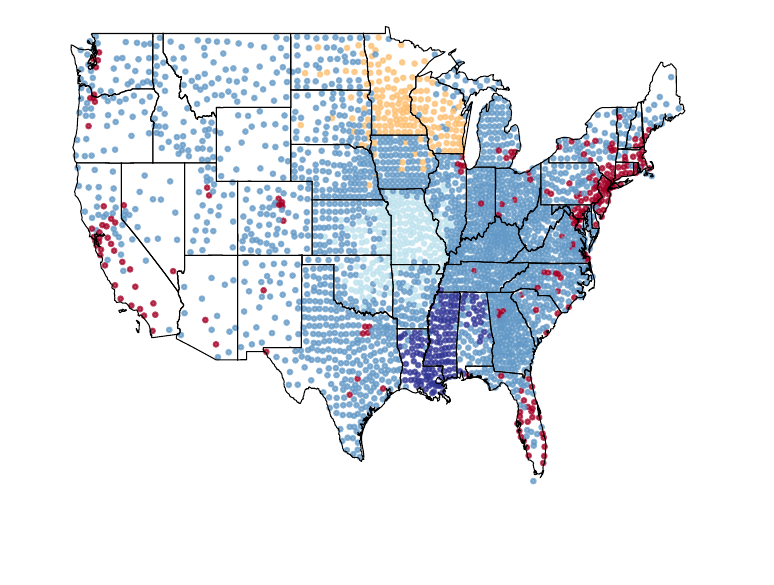}
        \vspace{-24pt}
        \caption{\texttt{LE-SC}}\label{fig:US_LESC}
    \end{subfigure}
    \begin{subfigure}{0.24\textwidth}
        \includegraphics[width = \textwidth]{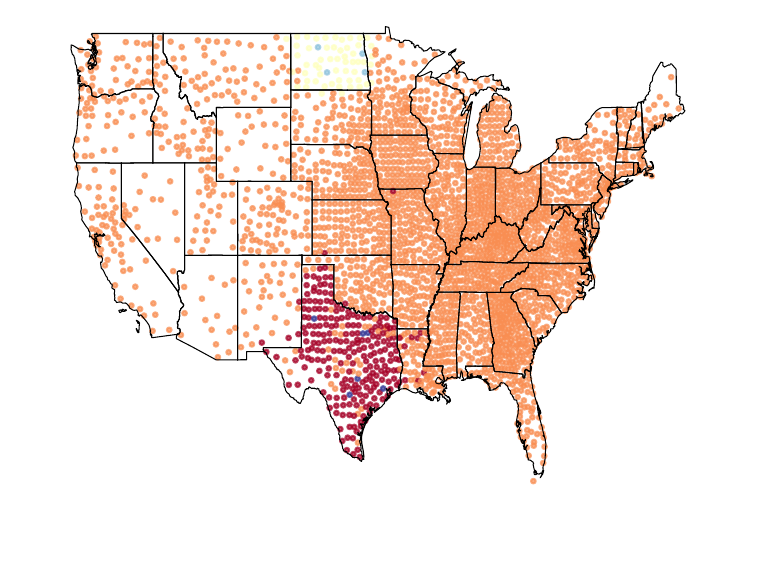}
        \vspace{-24pt}
        \caption{\texttt{Herm(RW)}}
    \end{subfigure}
    \begin{subfigure}{0.24\textwidth}
        \includegraphics[width = \textwidth]{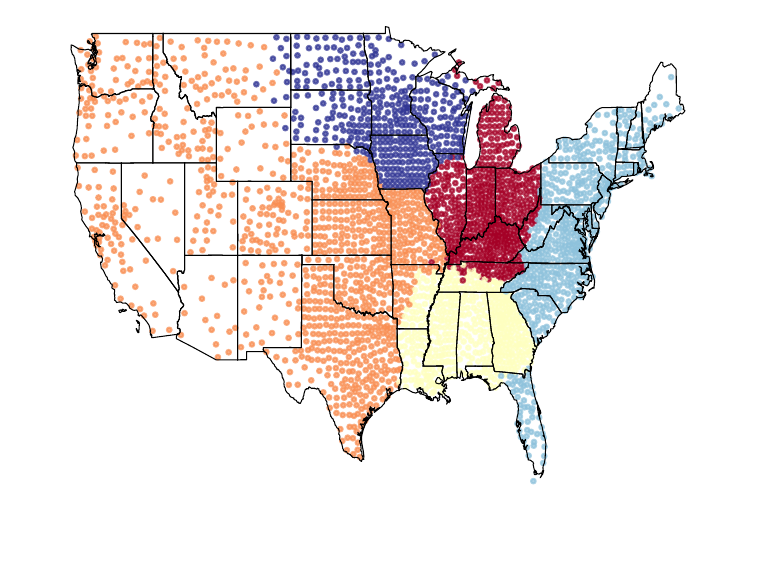}
        \vspace{-24pt}
        \caption{\texttt{DI-SIM(L)}}
    \end{subfigure}
    \begin{subfigure}{0.24\textwidth}
        \includegraphics[width = \textwidth]{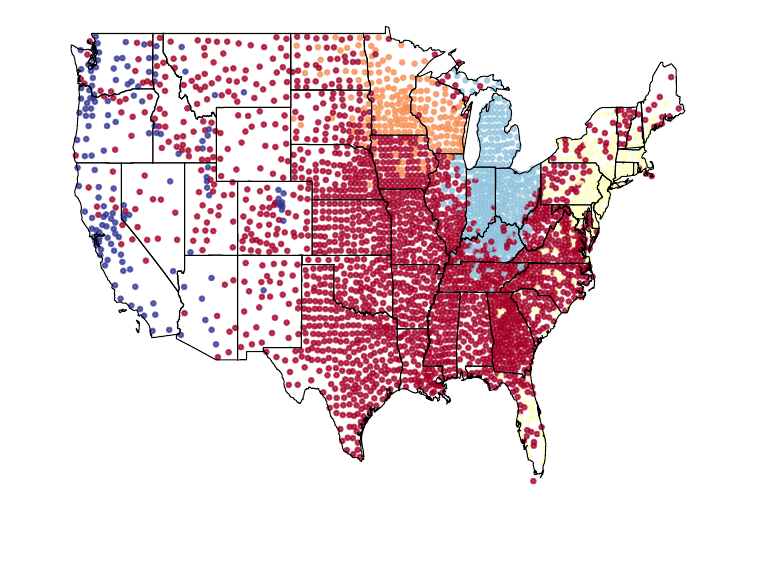}
        \vspace{-24pt}
        \caption{\texttt{Bib-Sym}}
    \end{subfigure}
    \begin{subfigure}{0.24\textwidth}
        \includegraphics[width = \textwidth]{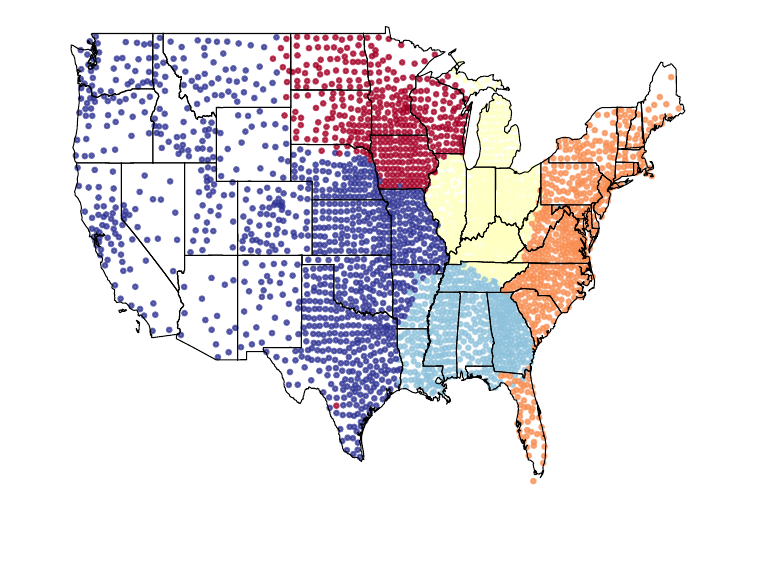}
        \vspace{-24pt}
        \caption{\texttt{DI-SIM(R)}}
    \end{subfigure}
    \begin{subfigure}{0.24\textwidth}
        \includegraphics[width = \textwidth]{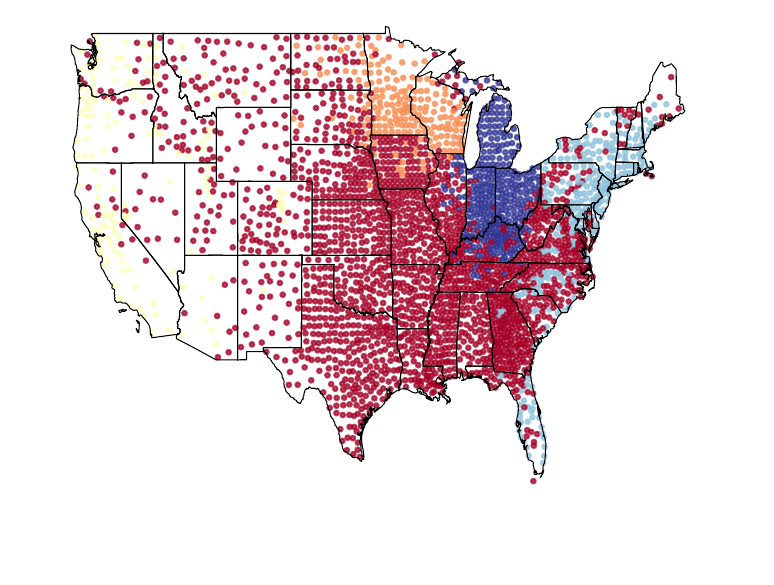}
        \vspace{-24pt}
        \caption{\texttt{Sym}}
    \end{subfigure}
    \begin{subfigure}{0.24\textwidth}
        \includegraphics[width = \textwidth]{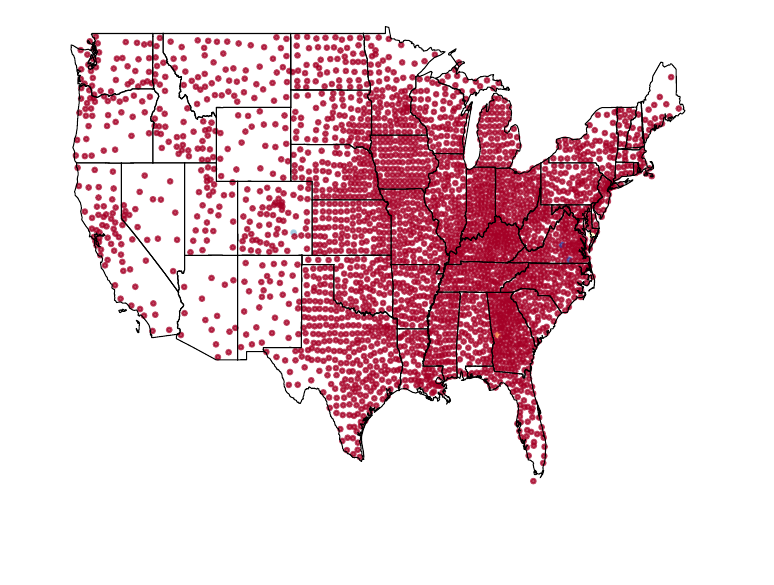}
        \vspace{-24pt}
        \caption{\texttt{SimpHerm}}
    \end{subfigure}
    \begin{subfigure}{0.24\textwidth}
        \includegraphics[width = \textwidth]{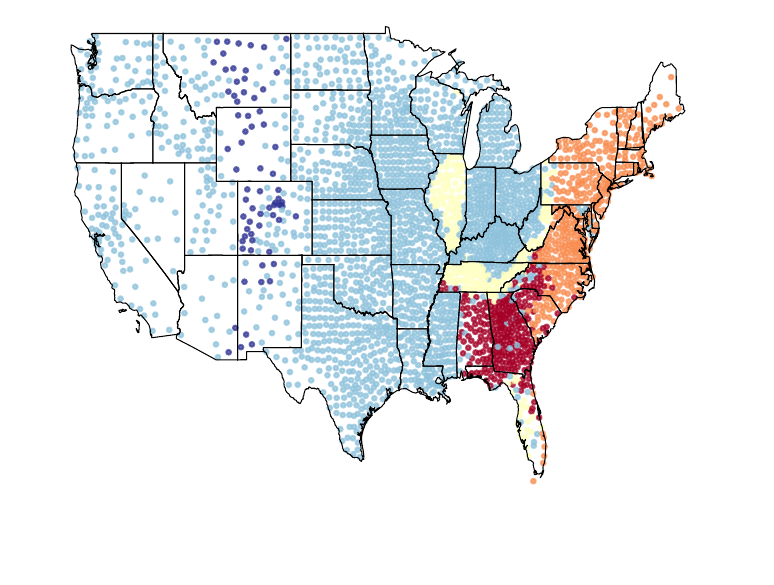}
        \vspace{-24pt}
        \caption{\texttt{D-Score}}
    \end{subfigure}
    \caption{Counties clustered by migration data ($k=5$).}
    \label{fig:migration5}
\end{figure}

\begin{figure}[tph!]
    \centering
    \begin{subfigure}{0.24\textwidth}
        \includegraphics[width = \textwidth]{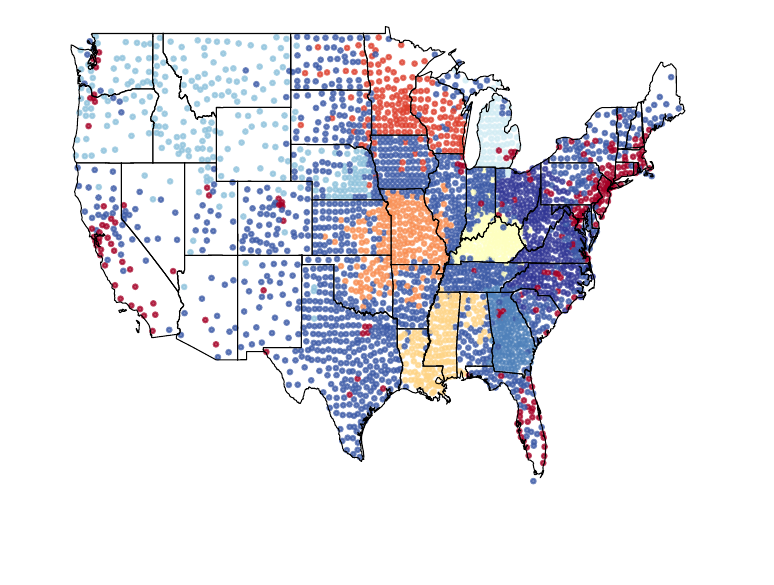}
        \vspace{-24pt}
        \caption{\texttt{LE-SC}}\label{fig:US_LESC}
    \end{subfigure}
    \begin{subfigure}{0.24\textwidth}
        \includegraphics[width = \textwidth]{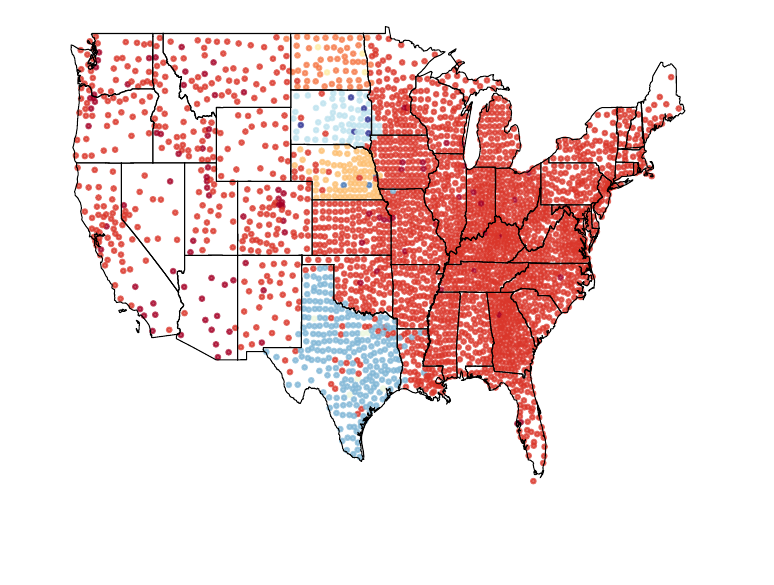}
        \vspace{-24pt}
        \caption{\texttt{Herm(RW)}}
    \end{subfigure}
    \begin{subfigure}{0.24\textwidth}
        \includegraphics[width = \textwidth]{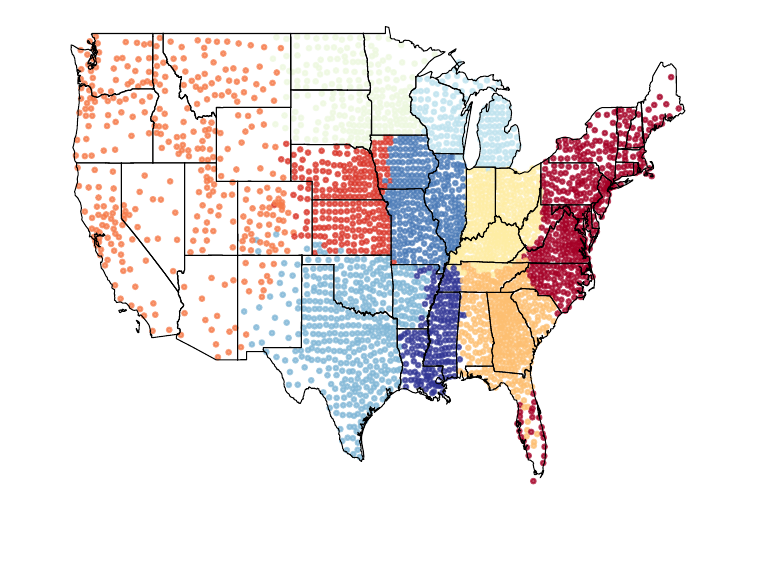}
        \vspace{-24pt}
        \caption{\texttt{DI-SIM(L)}}
    \end{subfigure}
    \begin{subfigure}{0.24\textwidth}
        \includegraphics[width = \textwidth]{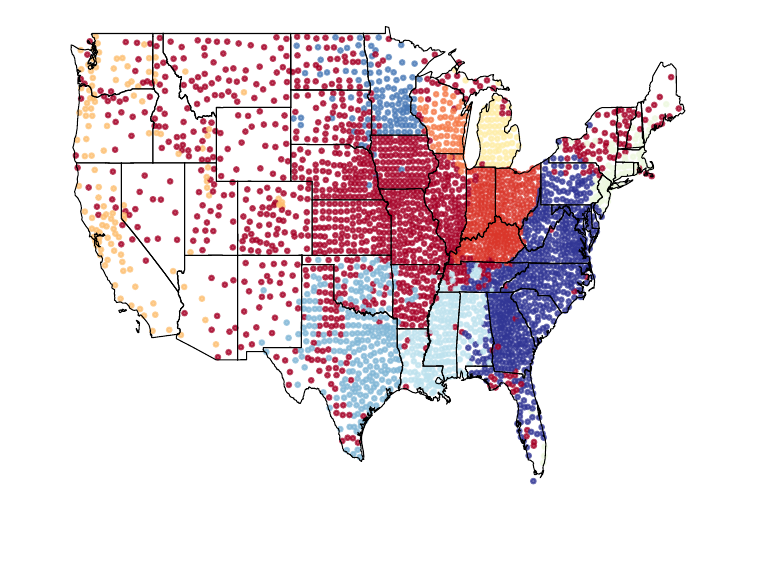}
        \vspace{-24pt}
        \caption{\texttt{Bib-Sym}}
    \end{subfigure}
    \begin{subfigure}{0.24\textwidth}
        \includegraphics[width = \textwidth]{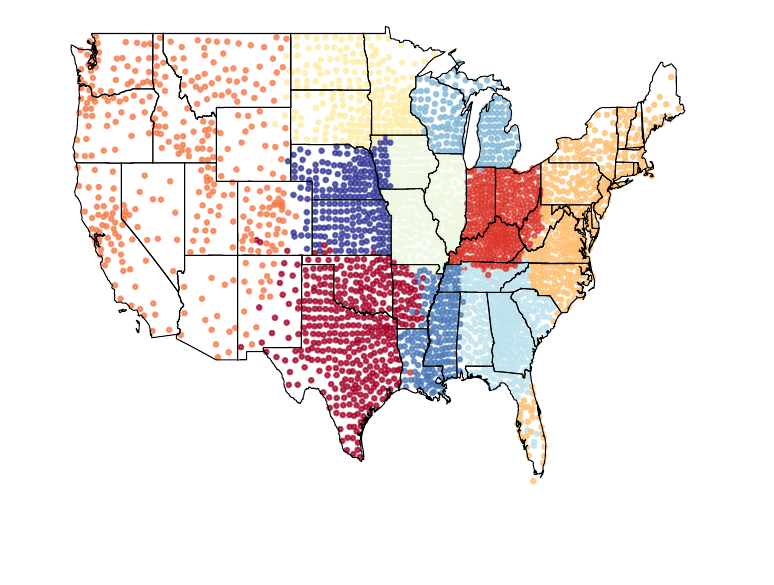}
        \vspace{-24pt}
        \caption{\texttt{DI-SIM(R)}}
    \end{subfigure}
    \begin{subfigure}{0.24\textwidth}
        \includegraphics[width = \textwidth]{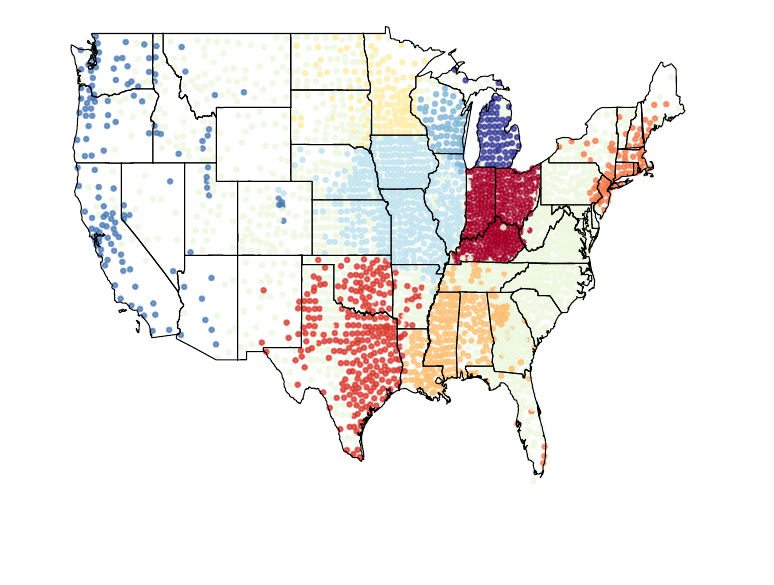}
        \vspace{-24pt}
        \caption{\texttt{Sym}}
    \end{subfigure}
    \begin{subfigure}{0.24\textwidth}
        \includegraphics[width = \textwidth]{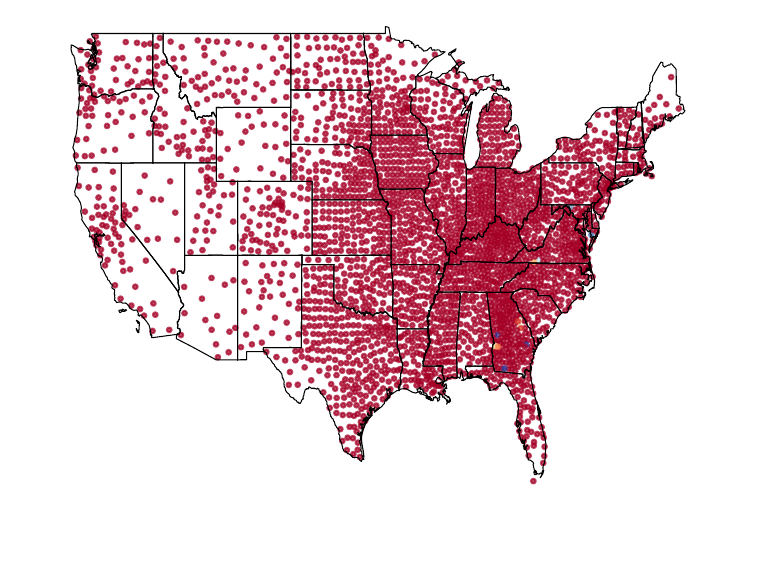}
        \vspace{-24pt}
        \caption{\texttt{SimpHerm}}
    \end{subfigure}
    \begin{subfigure}{0.24\textwidth}
        \includegraphics[width = \textwidth]{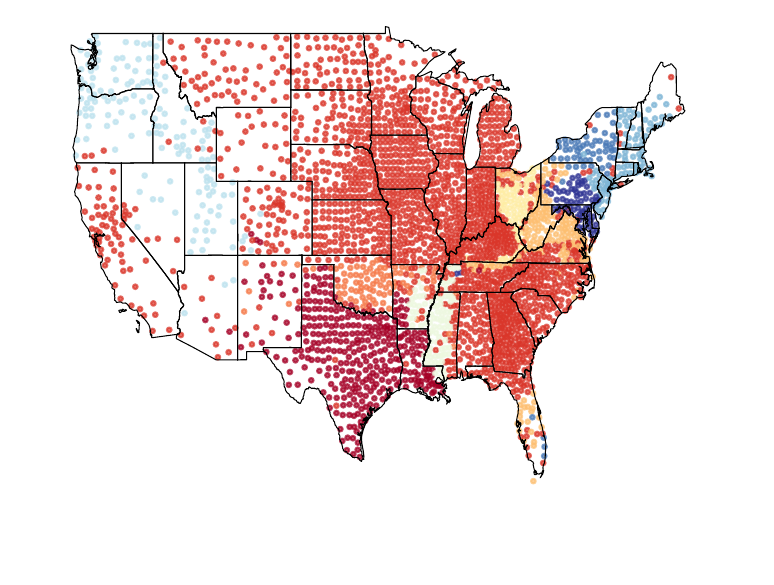}
        \vspace{-24pt}
        \caption{\texttt{D-Score}}
    \end{subfigure}
    \caption{Counties clustered by migration data ($k=10$).}
    \vspace{-12pt}
    \label{fig:migration10}
\end{figure}
